\documentclass{article}
\usepackage{natbib}
\usepackage{hypernat}
\usepackage{amsthm,amsmath,amssymb,amsfonts,amsxtra,verbatim,graphicx}
\usepackage{float}
\usepackage{times}
\usepackage{fullpage}
\RequirePackage[colorlinks,citecolor=blue,urlcolor=blue]{hyperref}
\RequirePackage{hypernat}

    \numberwithin{equation}{section}

    \theoremstyle{plain}
    \newtheorem{theorem}{Theorem}[section]

    \newtheorem{lemma}[theorem]{Lemma}

    \theoremstyle{definition}
    \newtheorem{definition}[theorem]{Definition}

    \theoremstyle{remark}
    \newtheorem{remark}[theorem]{Remark}

\newcommand{\argmin}{\operatornamewithlimits{argmin}}
\newcommand{\tr}{\operatornamewithlimits{tr}}
\newcommand{\logdet}{\operatornamewithlimits{logdet}}
\newcommand{\Var}{\operatornamewithlimits{var}}

\newcommand{\diag}{\operatornamewithlimits{diag}}
\newcommand{\sign}{\operatornamewithlimits{sign}}
\newcommand{\card}{\operatornamewithlimits{card}}
\newcommand{\prob}{\operatornamewithlimits{pr}}
\newcommand{\E}{\operatornamewithlimits{E}}

\newcommand{\T}[1]{\ensuremath{{#1}^{\mbox{\sf\tiny T}}}}

\begin{document}
%%%%%%%%%%%%%%%%%%%%%%%%%%%%%%%%%%%%%%%%%%%%%%%%%%%%%%%%%%%%%%%%%%%%%%

%\jname{Biometrika}
%%% The year, volume, and number are determined on publication
%\jyear{2009}
%\jvol{}
%\jnum{}
%%% The \doi{...} and \accessdate commands are used by the production team
%%\doi{10.1093/biomet/asm023}
%\accessdate{}
%\copyrightinfo{\Copyright\ 2009 Biometrika Trust\goodbreak {\em Printed in Great Britain}}
%
%%% These dates are usually set by the production team
%\received{July 2009}
%\revised{}

%%% The left and right page headers are defined here:
%\markboth{A. Shojaie \and G. Michailidis}{Penalized Likelihood Estimation of DAGs}

\title{Penalized Likelihood Methods for Estimation of Sparse High Dimensional Directed Acyclic Graphs}

\author{Ali Shojaie and George Michailidis \\
        Department of Statistics, University of Michigan, Ann Arbor \\
        \texttt{shojaie@umich.edu}, \texttt{gmichail@umich.edu}}

\date{}

%%%%%%%%%%%%%%%%%%%%%%%%%%%%%%%%%%%%%%%%%%

\maketitle
%\thispagestyle{empty}

%%%%%%%%%%%%%%%%%%%%%%%%%%%%%%%%%%%%%%%%%%
\begin{abstract}
  Directed acyclic graphs (\textsc{dag}s) are commonly used to represent causal relationships among random variables in graphical models. Applications of these models arise in the study of physical, as well as biological systems, where directed edges between nodes represent the influence of components of the system on each other. The general problem of estimating \textsc{dag}s from observed data is computationally NP-hard, Moreover two directed graphs may be observationally equivalent. When the nodes exhibit a natural ordering, the problem of estimating directed graphs reduces to the problem of estimating the structure of the network. In this paper, we propose a penalized likelihood approach that directly estimates the adjacency matrix of \textsc{dag}s. Both lasso and adaptive lasso penalties are considered and an efficient algorithm is proposed for estimation of high dimensional \textsc{dag}s. We study variable selection consistency of the two penalties when the number of variables grows to infinity with the sample size. We show that although lasso can only consistently estimate the true network under stringent assumptions, adaptive lasso achieves this task under mild regularity conditions. The performance of the proposed methods are compared to alternative methods in simulated, as well as real, data examples.
\end{abstract}
%%%%%%%%%%%%%%%%%%%%%%%%%%%%%%%%%%%%%%%%%%

%\begin{keywords}
%Directed Acyclic Graph; Penalized Likelihood Estimation; Lasso; Adaptive Lasso; High Dimensional Sparse Graphs; Small n Large p Asymptotics.
%\end{keywords}

%%%%%%%%%%%%%%%%%%%%%%%
\section{Introduction}\label{intro}
 Graphical models provide efficient tools for the study of statistical models through a compact representation of the joint probability distribution of the underlying random variables. The nodes of the graph represent the random variables, while the edges capture the relationships among them. Both directed and undirected edges are used to represent interactions among random variables. However, there is a conceptual difference between these two types of edges: while undirected edges are used to represent similarity or correlation, directed edges are usually interpreted as causal relationships. The study of directed edges is therefore directly related to the theory of causality, and of main interest in many applications. A special class of directed graphical models (also known as Bayesian Networks) are based on directed acyclic graphs (\textsc{dag}s), where all the edges of the graph are directed and there are no directed cycles present in the graph. \textsc{dag}s are used in graphical models and belief networks and have been the focus of research in the computer science literature (see \citet{pearl2000caus}). Important applications involving \textsc{dag}s also arise in the study of biological systems, as many cellular mechanisms are known to include causal relationships. Cell signalling pathways and gene regulatory networks are two examples, where causal relationships play an important role \citep{markowetz2007inferring}.

 The problem of estimating \textsc{dag}s is an NP-hard problem, and estimation of direction of edges may not be possible due to observational equivalence (see section \ref{DAGproblem}). Most of the earlier methods for estimating \textsc{dag}s correspond to greedy search algorithms that search through the space of possible \textsc{dag}s. A number of methods are available for estimating the structure of \textsc{dag}s for small to moderate number of nodes. The max-min hill climbing algorithm \citep{tsamardinos2006mmh}, and the PC-Algorithm \citep{spirtes2000cpa} are two such examples. However, the space of possible \textsc{dag}s grows super-exponentially with the number of variables (nodes), and estimation of \textsc{dag}s using these methods, especially in a small $n$, large $p$ setting, becomes impractical. Bayesian methods of estimating \textsc{dag}s \citep[e.g][]{heckerman1995lbn} are also computationally very intensive and therefore not particularly appropriate for large graphs. \citet{kalisch2007ehd} recently proposed an implementation of the PC-Algorithm with polynomial complexity that can be used for estimation of high dimensional sparse \textsc{dag}s. However, when the variables inherit a \emph{natural ordering}, estimation of a \textsc{dag} is reduced to estimating its structure or skeleton. Applications with natural ordering of variables include estimation of causal relationships from temporal observations, or settings where additional experimental data can determine the ordering of variables, and estimation of transcriptional regulatory networks from gene expression data. Examples of these applications are presented in section \ref{RealData}.

 The structure of the graph can be determined from conditional independence relations among random variables. For undirected graphs, this is equivalent to learning the structure of the conditional independence graph (\textsc{cig}), which in the case of Gaussian random variables, is determined by zeros in the inverse covariance matrix (also known as precision or concentration matrix). Different penalization methods, have been recently proposed to obtain sparse estimates of the concentration matrix. \citet{meinshausen2006hdg} considered an approximation to the problem of sparse inverse covariance estimation using the lasso penalty. They showed under a set of assumptions, that their proposed method correctly determines the neighborhood of each node. \citet{banerjee2008mst} and \citet{friedman2008sic} explored different aspects of the problem of estimating the concentration matrix using the lasso penalty, while \citet{yuan2007msa} and \citet{fan2007nev} considered other choices for the penalty. \citet{RothEtal:08} proved consistency in Frobenius norm, as well as in matrix norm, of the $\ell_1$-regularized estimate of the concentration matrix when $p \gg n$, while \citet{lam2008src} extended their result and considered estimation of matrices related to the precision matrix, including the Cholesky factor of the inverse covariance matrix, using general penalties. Penalization of the Cholesky factor of the inverse covariance matrix has been also considered by \citet{huang2006cms}, where they used the lasso penalty in order to obtain a sparse estimate of the inverse covariance matrix. This method is based on the regression interpretation of the Cholesky factorization model and therefore requires the variables to be ordered \emph{a priori}.

 In this paper, we consider the problem of estimating the skeleton of \textsc{dag}s, where the variables exhibit a natural ordering. We use graph theoretic properties of \textsc{dag}s and reformulate the likelihood as a function of the adjacency matrix of the graph. We then exploit the ordering of variables to propose an efficient algorithm for estimation of structure of \textsc{dag}s, which offers considerable improvement in terms of computational complexity. Both lasso and adaptive lasso penalties are considered and variable selection consistency of estimators is established in the $p \gg n$ setting. In particular, we show that although lasso is only variable selection consistent under stringent conditions, adaptive lasso can consistently estimate the true \textsc{dag} under the usual regularity assumptions. We also present a data dependent choice of the tuning parameter that controls the probability of errors. Theoretical as well as empirical evidence shows that when the underlying causal mechanism in the network is linear, the proposed method can also be applied to non-Gaussian observations. Finally, additional simulations indicate that although the proposed method is derived based on the ordering of variables, the method is not sensitive to random permutations of the order of variables in high dimensional sparse settings.

%%%%%%%%%%%%%%%%%%%%%%%
\section{Representation of Directed Acyclic Graphs}\label{DAGproblem}
 Consider a graph $\mathcal{G} = (V,E)$, where $V$ corresponds to the set of nodes with $p$ elements and $E \subset V \times V$ to the edge set. The nodes of the graph represent random variables $X_1, \ldots, X_p$ and the edges capture associations amongst them. An edge is called directed if $(i,j) \in E \Rightarrow (j,i) \notin E$ and undirected if $(i,j) \in E \Rightarrow (j,i) \in E$. The main focus of this paper is a special class of graphs where $E$ consists of only directed edges, and does not include directed cycles. We denote by $pa_i$ the set of parents of node $i$ and for $j \in pa_i$, we denote $j \rightarrow i$. The \emph{skeleton} of a \textsc{dag} is the undirected graph that is obtained by replacing directed edges in $E$ with undirected ones. Finally, throughout this paper, we represent $E$ using the adjacency matrix $A$ of the graph; i.e. a $p\times p$ matrix whose $(j,i)$th entry indicates whether there is an edge (and possibly its weight) between nodes $j$ and $i$.
%  A probability distribution $\mathcal{P}$ on $\mathbb{R}^p$ is said to be \emph{(Markov) compatible} with a \textsc{dag} $\mathcal{G}$ if $\mathcal{P}$ admits the following decomposition relative to $\mathcal{G}$
%    \begin{equation}\label{eqnProbDecomp}
%        \mathcal{P}(x_1, \ldots, x_p) = \prod_i{\mathcal{P}(x_i | pa_i)}
%    \end{equation}

 The estimation of \textsc{dag}s is a challenging problem due to the so-called \emph{observational equivalence} of \textsc{dag}s with respect to the same probability distribution. More specifically, regardless of the sample size, it may not be possible to infer the direction of causation among random variables from observational data. As an illustration of the observational equivalence, consider the simple \textsc{dag} in the right panel of Figure \ref{ToyEx}. Reversing the direction of all edges of the graph results in a new \textsc{dag}, which is the same as the original graph, except for changes in the node labels and is therefore \emph{polymorphic} to the original one. It is therefore natural to estimate the \emph{equivalence class} of \textsc{dag}s corresponding to the same probability distribution $\mathcal{P}$ starting with the skeleton of the network.
   \begin{figure}[t!]
       \centering
       \scalebox{0.5}
       {\includegraphics [clip=TRUE, trim=0cm 0cm 0cm 0cm ]{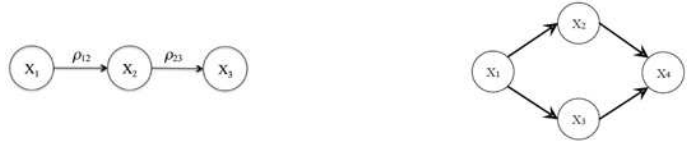}} %Crops off from left, bottom, right, top
        \caption{Left: A simple \textsc{dag}, Right: Illustration of observational equivalence in \textsc{dag}s}
        \label{ToyEx}
   \end{figure}
%   \begin{figure}[h!]
%        \figurebox{1.75cm}{9cm}{}[ToyExs.eps]
%        \caption{Left: A simple \textsc{dag}, Right: Illustration of observational equivalence in \textsc{dag}s}
%        \label{ToyEx}
%    \end{figure}
 The second challenge in estimating \textsc{dag}s is that conditional independence among random variables may not reveal the skeleton. The notion of conditional independence in \textsc{dag}s is either represented using the concept of \emph{d-separation} \citep{pearl2000caus} or the \emph{moral graph} of a \textsc{dag} \citep{lauritzen1996gm}. The moral graph is obtained by removing the directions of the graph and ``marrying'' the parents of each node. Therefore estimation of the conditional independence structure reveals the structure of the moral graph of the \textsc{dag}, which includes additional edges between parents of each node. This can also be illustrated using the simple graph in the left panel of Figure \ref{ToyEx}. Suppose $X_i, i=1, \ldots, 4$ are normally distributed with covariance matrix $\Sigma$. The only zero elements of the inverse covariance matrix are $\Sigma^{-1}_{14} = \Sigma^{-1}_{41}$, as $X_2$ and $X_3$ are connected in the moral graph of $\mathcal{G}$.

\subsection{The Latent Variable Model}\label{latent}
 The causal effect of random variables in a directed acyclic graph can be explained using \emph{structural equation models}, where each variable is modeled as a (nonlinear) function of its parents. The general form of these models is given by \citep{pearl2000caus}:
 \begin{equation}\label{eqnSEM}
   X_i = f_i(pa_i,Z_i), \hspace{0.5cm} i=1, \ldots, p
 \end{equation}
 The random variables $Z_i$ are the latent variables representing the unexplained variation in each node. To model the association among nodes of a \textsc{dag}, we consider a simplification of (\ref{eqnSEM}) with $f_i$ being linear. More specifically, let $\rho_{ij}$ represent the \emph{effect} of node $j$ on $i$ for $j \in pa_i$, then
 \begin{equation}\label{eqnSEMlin}
   X_i = \sum_{j \in pa_i}{\rho_{ij} X_j} + Z_{i}, \hspace{0.5cm} i=1, \ldots, p
 \end{equation}
 In the special case where the random variables are Gaussian, equations (\ref{eqnSEM}) and (\ref{eqnSEMlin}) are equivalent in the sense that $\rho_{ij}$ are coefficients of the linear regression model of $X_i$'s on $X_j, j \in pa_i$. It is known in the normal case that $\rho_{ij}=0, j \notin pa_i$.
 
 Consider the simple \textsc{dag} in the left panel of Figure \ref{ToyEx}; denoting the \emph{influence matrix} of the graph by $\Lambda$, (\ref{eqnSEMlin}) can be written in compact form as $X = \Lambda Z$, where for the simple example above, we have
  \[ \Lambda = \left( \begin{array}{ccc}
    1 & 0 & 0 \\
    \rho_{12} & 1 & 0 \\
    \rho_{12} \rho_{23} & \rho_{23} & 1
  \end{array} \right)\]
  \vspace{0.2cm}
 \noindent Let the latent variables $Z_i$ be independent with mean $\mu_i$ and variance $\sigma^2_i$. Then $\E(X) = \Lambda \mu$ and $\Sigma = \Var(X) = \Lambda D \T{\Lambda}$, where $D = \diag{(\sigma^2_i)}$ and $\T{\Lambda}$ denotes the transpose of the matrix $\Lambda$.

 The following result from \citet{shojaie2009NetBasedGSA} establishes relationships between the influence matrix $\Lambda$, and the adjacency matrix of the graph, $A$. The second part of the lemma establishes a compact relationship between $\Lambda$ and $A$ in the case of \textsc{dag}s, which is explored in section \ref{DAGest} to directly formulate the problem of estimating the skeleton of a \textsc{dag}.
 \begin{lemma}\label{lmLambdaA}
   For any graph $\mathcal{G} = (V, A)$,
  \begin{enumerate}
    \item[(i)]
        $\Lambda = A^0 + A^1 + A^2 + \cdots = \sum_{r=0}^{\infty}{A^r} $, where $A^0 \equiv I$.
%    \item[(ii)]
%        If $\mathcal{G}$ is a \textsc{dag}, $\Lambda = A^0 + A^1 + A^2 + \cdots + A^p$.
    \item[(ii)]
        If $\mathcal{G}$ is a \textsc{dag}, $\Lambda$ has full rank and $\Lambda = (I - A)^{-1}$.
  \end{enumerate}
 \end{lemma}
 \begin{remark}\label{remSigma}
   Part (ii) of Lemma \ref{lmLambdaA} and the fact that $\Sigma = \Lambda D \T{\Lambda}$ imply that for any \textsc{dag}, if $D_{ii} > 0$ for all $i$, then $\Sigma$ is full rank. More specifically, let $\phi_j(M)$ denote the $j$th eigenvalue of matrix $M$. Then, $\phi_{\min}(\Sigma) > 0$ (or $\phi_{\max}(\Sigma^{-1}) < \infty$). Similarly, since $\Sigma^{-1} = \T{\Lambda^{-}} D^{-1} \Lambda^{-1}$, full rankness of $\Lambda$ implies that $\phi_{\min}(\Sigma^{-1}) > 0$  (or equivalently $\phi_{\max}(\Sigma) < \infty$). This result also applies to all subnetworks of a \textsc{dag}.
 \end{remark}
 The properties of the proposed latent variable model established in Lemma \ref{lmLambdaA} are independent of the choice of probability distribution $\mathcal{P}$. In fact, since the latent variables $Z_i$ in (\ref{eqnSEMlin}) are assumed independent, given the entries of the adjacency matrix, the distribution of each random variable $X_i$ in the graph only depends on the values of $pa_i$. Therefore, regardless of the choice of the probability distribution, the joint distribution of the random variables is compatible with $\mathcal{G}$ (see for example \citet{pearl2000caus} p. 16). 
% Therefore, based on the equivalence of conditional independence and d-separation in \textsc{dag}s, if the joint probability distribution of random variables on a \textsc{dag} is generated according to the latent variable model (\ref{eqnSEMlin}), zeros in the adjacency matrix $A$, determine conditional independence relations among random variables.
 In section \ref{PerfAnal}, we illustrate this result using data generated according to non-Gaussian distributions.

%%%%%%%%%%%%%%%%%%%%%%%
\section{Penalized Estimation of \textsc{dag}s}\label{DAGest}

\subsection{Problem Formulation}\label{problem}
  Consider the latent variable model of section \ref{latent} and denote by $\mathcal{X}$ the $n \times p$ data matrix. We assume, without loss of generality, that the $X_i$'s are centered and scaled, so that $\mu_i = 0$ and $\sigma^2_i = 1, i=1, \ldots, p$. Note that the results in section \ref{latent} were established independent of the choice of the probability distribution. As mentioned before, under the normality assumption, the latent variable model is equivalent to the general structural equation model. Although we focus on Gaussian random variables in the remainder of this paper, the estimation procedure proposed in this section can be applied to a variety of other distributions, if one is willing to assume the linear structure in (\ref{eqnSEMlin}).

  Denote by $\Omega\equiv\Sigma^{-1}$ the precision matrix of a $p$-vector of Gaussian random variables and consider a general penalty function by $J(\Omega)$. The penalized likelihood function is then given by
    \begin{equation}\label{eqnCovEst}
       \hat{\Omega} = \argmin_{\Omega \succ 0}{  \left\{
                                            - \logdet{(\Omega)} + \tr{(\Omega S)}+ \lambda J(\Omega)
                                          \right\}   }
    \end{equation}
  \noindent where $S = n^{-1}\T{\mathcal{X}} \mathcal{X}$ denotes the empirical covariance matrix and $\lambda$  is the tuning parameter controlling the size of the penalty. Applications in biological and social networks often involve sparse networks. It is therefore desirable to find a sparse solution for (\ref{eqnCovEst}). This becomes more important in the small $n$, large $p$ setting, where the unpenalized solution includes many additional edges. The lasso penalty of \citet{tibshirani1996rss} and the adaptive lasso penalty proposed by \citet{zou2006ala} are singular at zero and therefore result in sparse solutions. We consider these two penalties in order to find a sparse estimate of the adjacency matrix. Other choices of the penalty function are briefly discussed in the conclusions section.

  Using the latent variable model of section \ref{latent}, and the relationship between the covariance matrix and the adjacency matrix of \textsc{dag}s established in Lemma \ref{lmLambdaA}, the problem of estimating the adjacency matrix of the graph can be directly formulated as an optimization problem based on $A$. As noted in Lemma \ref{lmLambdaA}, if the underlying graph is a \textsc{dag} and the ordering of the variables is known, then $A$ is a lower triangular matrix with zeros on the diagonal. Let $\mathcal{A} = \{A: A_{ij}=0, \thickspace j \ge i\}$. Then using the facts that $\det(A) = 1$ and $\sigma^2_i = 1$,  $A$ can be estimated as the solution of the following optimization problem
     \begin{equation} \label{AOpt_DAG}
       \hat{A} = \argmin_{A \in \mathcal{A}} {  \left\{
                                    \tr{ \left[ \T{(I-A)} (I-A)S \right] } + \lambda J(A)
                                \right\}  }
     \end{equation}
  In this paper, we consider the general weighted lasso problem, where
  \begin{equation}\label{eqn_penalty}
    J(A) = \lambda \sum_{i,j=1:p, \smallskip j<i}{w_{ij} |A_{ij}|}
  \end{equation}
  Lasso and adaptive lasso problems are special cases of this general penalty. In the case of lasso, $w_{ij} = 1$. The original weights in adaptive lasso, proposed by \citet{zou2006ala} are obtained by setting $w_{ij} = |\tilde{A}_{ij}|^{-\gamma}$, for some initial estimate of the adjacency matrix $\tilde{A}$ and some power $\gamma$. To facilitate the study of asymptotic properties of adaptive lasso estimates we consider the following modification of the original weights
  \begin{equation}\label{eqn_weight}
    w_{ij} = 1 \vee |\tilde{A}_{ij}|^{-\gamma}
  \end{equation}
  \noindent where the original estimates $\tilde{A}$ are obtained from the regular lasso estimates.

  The objective function for both lasso and adaptive lasso problems is convex. However, since the $\ell_1$ penalty is non-differentiable, these problems can be reformulated using matrices $A_{+} = \max(A , 0)$ and $A_{-} = - \min(A , 0)$. To that end,
  let $W$ be the $p \times p$ matrix of weights for adaptive lasso, or the matrix of ones for the lasso estimation problem. Problem (\ref{AOpt_DAG}) can then be formulated as:
    \begin{equation}\label{AOpt_DAGpn2}
       \min_{A_{+} , A_{-} \succeq 0} \tr{ \left\{ \T{S(I - A_{+} + A_{-})} (I - A_{+} + A_{-}) + \lambda(A_{+} + A_{-})W + \Delta (A_{+} + A_{-}) \boldsymbol{1}_{l^{+}} \right\} }
    \end{equation}
  \noindent where $\succeq 0$ is interpreted componentwise, $\Delta$ is a large positive number and $\boldsymbol{1}_{l^{+}}$ is the indicator matrix for lower triangular elements of a $p \times p$ matrix, including the diagonal elements. The last term of the objective function (${\tr[\Delta (A_{+} + A_{-}) \boldsymbol{1}_{l^{+}}]}$) prevents the upper triangular elements of the matrices $A_{+}$ and $A_{-}$ to be nonzero.

  Problem (\ref{AOpt_DAGpn2}) is a quadratic optimization problem with non-negativity constraints and can be solved using standard interior point algorithms. However, such algorithms do not scale well with dimension and are only applicable if $p$ ranges in the hundreds. In section \ref{algDAG}, we present an alternative formulation of the problem, which leads to considerably more efficient algorithms.

\subsection{Optimization Algorithm}\label{algDAG}
  Consider again the problem of estimating the adjacency matrix of \textsc{dag}s with either lasso or adaptive lasso penalties. Denoting the $i$th \emph{row} of matrix $A$ as $A_i$ we can write (\ref{AOpt_DAG}) as:
     \begin{equation} \label{AOpt_DAG_sep1}
       \hat{A} = \argmin_{A \in \mathcal{A}} {  \left\{
                \sum_{i=1}^{p}{ \T{A_{i}} S A_{i} } - 2 A_{i} \T{S_{i}} + \lambda \T{W_i}|A_{i}|
                \right\}  }
     \end{equation}
  It can be seen that the objective function in (\ref{AOpt_DAG_sep1}) is separable and therefore it suffices to solve the optimization problem over each row of matrix $A$.
%  For any matrix, $M$, denote by $M_{\mathcal{I},\mathcal{J}}$ the submatrix with rows indexed by $\mathcal{I}$ and columns indexed by $\mathcal{J}$, and denote by $M_{\mathcal{J}}$ the submatrix obtained from columns of $M$ indexed by $\mathcal{J}$. Further denote
  Denote by $\underline{l}$ the set of indices up to $l$, i.e. $\underline{l} = j: 1 \le j \le l$.

  Then, taking advantage of the lower triangular structure of $A$, solving (\ref{AOpt_DAG_sep1}) is equivalent to solving the following $p-1$ optimization problems ($A_{11} = 0$)
     \begin{equation} \label{AOpt_DAG_sep2}
       \hat{A}_{i,\underline{i-1}} = \argmin_{ \theta \in \mathbb{R}^{i-1} } {  \left\{
                                    \T{\theta} S_{\underline{i-1},\underline{i-1}} \theta -
                                    2 S_{i,\underline{i-1}} \theta + \lambda \sum_{j=1}^{i-1}{|\theta_j| w_{ij}} \right\}  },
                                    \hspace{1cm} i=2, \ldots, p
     \end{equation}
  \noindent Using the facts that $S_{\underline{i-1},\underline{i-1}} = n^{-1} \T{(\mathcal{X}_{\underline{n},\underline{i-1}})} \mathcal{X}_{\underline{n},\underline{i-1}}$ and $S_{i,\underline{i-1}} = n^{-1} \T{(\mathcal{X}_{\underline{n},i})} \mathcal{X}_{\underline{n},\underline{i-1}}$, the problem in (\ref{AOpt_DAG_sep2}) can be reformulated as following $\ell_1$-regularized least squares problems
     \begin{equation} \label{AOpt_DAG_lasso}
       \hat{A}_{i,\underline{i-1}} = \argmin_{ \theta \in \mathbb{R}^{i-1} } { \left\{
              n^{-1}\| \mathcal{X}_{\underline{n},\underline{i-1}} \theta - \mathcal{X}_{\underline{n},i} \|_2^2 + \lambda_i \sum_{j=1}^{i-1}{|\theta_j| w_{ij}} \right\}  }, \hspace{1cm} i=2, \ldots, p
     \end{equation}
  \noindent The formulation in (\ref{AOpt_DAG_lasso}) indicates that the $i$th row of matrix $A$ includes the coefficient of projecting $X_i$ on $X_j, j = 1, \ldots, i-1$, which is in agreement with the discussion in section \ref{latent}. It also reveals a connection between covariance selection methods and the neighborhood selection approach of \citet{meinshausen2006hdg}; namely, when the underlying graph is a \textsc{dag}, the approximate solution of the neighborhood selection problem is exact, if the regression model is fitted on the set of parents of each node instead of all other nodes in the graph.
  \floatstyle{ruled}
  \newfloat{algorithm}{b}{loa}
  \floatname{algorithm}{Algorithm}
  \begin{algorithm}
    \caption[Algorithm]{Penalized Likelihood Estimation of \textsc{dag}s}
    \label{PLDAGestAlg}
    \begin{tabbing}
        1. \= Given the ordering $\mathcal{O}$, order the columns of observation matrix $\mathcal{X}$ in increasing order\\
        2. \= For $i = 2, 3, \ldots, p$\\
        \>2.1. \=
        Let $y = \mathcal{X}_{\underline{n},i}$, $X = \mathcal{X}_{\underline{n},\underline{i-1}}$ and $w = W_{i,\underline{i-1}}$ \\
        \>2.2. Given the weight matrix $W$, solve
        $\hat{A}_{i,\underline{i-1}} = \argmin { \left\{ n^{-1} \| X \theta - y \|_2^2 + \lambda_i \sum_{j=1}^{i-1}{|\theta_j| w_{j}} \right\}  }$
    \end{tabbing}
  \end{algorithm}

  Using (\ref{AOpt_DAG_lasso}), the problem of estimating \textsc{dag}s can be solved very efficiently. In fact, it suffices to solve $p-1$ lasso problems for estimation of least squares coefficients, with dimensions ranging from $1$ to $p-1$. To solve these problems, we use the efficient pathwise coordinate optimization algorithm of \citet{friedman2008rpg}, implemented in the R-package \texttt{glmnet}. The proposed algorithm is summarized in Algorithm \ref{PLDAGestAlg}.

\subsection{Analysis of Computational Complexity}\label{complexity}
  In this section, we provide a comparison of the computational complexity of the algorithm proposed in section \ref{algDAG} and the PC-Algorithm. As mentioned in the introduction, the space of all possible \textsc{dag}s is super-exponential in the number of nodes and hence it is not surprising that the PC-Algorithm, without any restriction on the space of \textsc{dag}s has exponential complexity. \citet{kalisch2007ehd} propose an efficient implementation of the PC-Algorithm for sparse \textsc{dag}s; its complexity where the maximal neighborhood size $q$ is small, is bounded with high probability by $O(p^{q})$. Although this is a considerable improvement over other methods of estimating \textsc{dag}s, in many applications it can become fairly expensive. For example, gene regulatory networks and signaling pathways exhibit a ``hub'' structure, which leads to large values for $q$. The graphical lasso algorithm, proposed by \citet{friedman2008sic}, uses an iterative algorithm for estimation of the inverse covariance matrix and has computational complexity $O(p^3)$.

  The reformulation of the \textsc{dag} estimation problem in (\ref{AOpt_DAG_lasso}) requires solving $p-1$ lasso regression problems. The cost of solving a lasso problem comprised of $k$ covariates and $n$ observations using the pathwise coordinate optimization (shooting) algorithm of \citet{friedman2008rpg} is $O(nk)$; hence, the total cost of estimating the adjacency matrix of the graph is $O(np^2)$, which is the same to the cost of calculating the (full) empirical covariance matrix $S_n$. Moreover, the formulation in (\ref{AOpt_DAG_lasso}) includes a set of non-overlapping sub-problems. Therefore, for problems with very large number of nodes and/or observations, the performance of the algorithm can be further improved by parallelizing the estimation of these sub-problems. The adaptive lasso version of the problem is similarly solved using the modification of the regular lasso problem proposed in \citet{zou2006ala}, which results in the same computational cost as the regular lasso problem.

  Figure \ref{figCPUtime}, compares the CPU time required for estimation of DAGs using both the PC-Algorithm, as well as our proposed algorithm for a range of values of $p$ and $n$. To control the complexity of the PC algorithm, the average neighborhood size is set to 5 and the significance level for the PC-Algorithm, as well as the tuning parameter for lasso and adaptive lasso penalties are set according to the optimal values discussed in Section \ref{PerfAnal}. The time reported for the PC-Algorithm is the CPU time required for estimation of the skeleton of the graph. The plot demonstrates the higher order of complexity of the PC-Algorithm, as well as the dependency of the algorithm on the sample size.
    \begin{figure}[t!]
        \centering
        \scalebox{0.85}
        {\includegraphics[clip=TRUE, trim=0cm 0cm 0cm 0cm ]{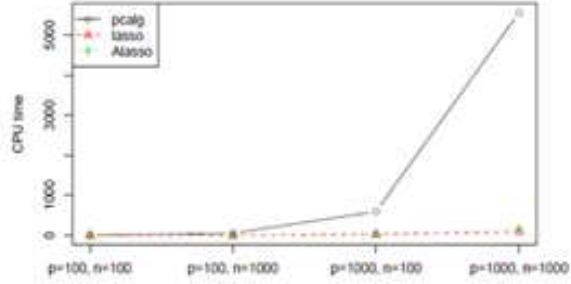}} %Crops off from left, bottom, right, top
        \caption{CPU time for analysis of simulated \textsc{dag}s with different number of nodes and sample size. The results are average times of $10$ repetitions on an Intel $2\cdot0$GH processor with $2\cdot0$GB of RAM.}
        \label{figCPUtime}
    \end{figure}
%    \begin{figure}[t!]
%        \figurebox{3cm}{12cm}{}[TimeComp112209.eps]
%        \caption{CPU times for estimation of \textsc{dag}s with different values of $p$ and $n$. The results are average times of 10 repetitions on an Intel 2.0GH processor with 2.0GB of RAM.}
%        \label{figCPUtime}
%    \end{figure}

%%%%%%%%%%%%%%%%%%%%%%%
\section{Asymptotic Properties}\label{asymptotics}
\subsection{Preliminaries}
  We next establish theoretical properties for both lasso, as well as adaptive lasso estimators of the adjacency matrix of \textsc{dag}s. Asymptotic properties of lasso-type estimates with fixed design matrices have been studied by a number of researchers \citep[see e.g.][]{knight2000alt, zou2006ala, huang2008als}. Lasso estimates with random design matrices have been also considered by \citet{meinshausen2006hdg}. On the other hand, \citet{RothEtal:08} and \citet{lam2008src} among others have studied asymptotic properties of estimates of covariance and precision matrices.

  As discussed in section \ref{problem}, the problem of estimating the adjacency matrix of a \textsc{dag} is equivalent to solving $p-1$ non-overlapping penalized least square problems described in (\ref{AOpt_DAG_sep2}). In order to study the asymptotic properties of the proposed estimators, we focus on the asymptotic consistency of network estimation, i.e. the probability of correctly estimating the network structure, in terms of type I and type II errors. We allow the total number of nodes in the graph to grow as an arbitrary polynomial function of the sample size, while assuming that the true underlying network is sparse. The assumptions required for establishing asymptotic properties of lasso and adaptive lasso estimates of \textsc{dag}s are presented in section \ref{asymp_assump}. We then study variable selection consistency of both lasso and adaptive lasso estimates in section \ref{asymp_results}. The choice of penalty parameter is studied in section \ref{tuning}. Technical proofs are given in the Appendix.

\subsection{Assumptions}\label{asymp_assump}
  Let $X = (X_1, \ldots, X_p)$ be a collection of $p$ zero-mean Gaussian random variables with covariance matrix $\Sigma$. Denote by $\mathcal{X}$ the $n \times p$ matrix of observations, and by $S$ the empirical covariance matrix. To simplify the notation, denote by $\theta^{i} = A_{i,\underline{i-1}}$ the entries of the $i$th row of $A$ to the left of the diagonal. Further, let $\theta^{i,\mathcal{I}}$ be the estimate for the $i$th row, with values outside the set of indices $\mathcal{I}$ set to zero; i.e.,
  $
    \theta^{i,\mathcal{I}} \equiv A_{i,\underline{i-1}} \text{  and  } A_{i,j} = 0, \thickspace j \notin \mathcal{I}
  $. Consider the following assumptions:
  \begin{enumerate}
    \item[(A-0)]
        For some $a > 0$, $p = p(n) = O(n^a)$ as $n \rightarrow \infty$, and there exists a $0 \le b < 1$ such that $\max_{i \in V} \card{(pa_i)} = O(n^b)$ as $n \rightarrow \infty$.
%\footnote{When there are no ambiguities, we suppress the dependency of $p=p(n)$ and the set of nodes in graph ($V = V(n)$) on $n$.}
    \item[(A-1)]
        There exists $\nu > 0$ such that for all $n \in \mathbb{N}$ and all $i \in V$, $\Var{\left( X_i \mid X_{\underline{i-1}} \right)} \ge \nu$.
    \item[(A-2)]
        There exists $\delta > 0$ and some $\xi > b$ (with $b$ defined above) such that for all $i \in V$ and for every $j \in pa_i$, $| \pi_{ij} | \ge \delta n^{-(1-\xi)/2}$, where $\pi_{ij}$ is the partial correlation between $X_i$ and $X_j$ after removing the effect of the remaining variables.
    \item[(A-3)]
        There exists $\Psi < \infty$ such that for all $n \in \mathbb{N}$ and every $i \in V$ and for every $j \in pa_i$, $\| \theta^{j,pa_i} \| \le \Psi$.
    \item[(A-4)]
        There exists $\kappa < 1$ such that for all $i \in V$ and for every $j \notin pa_i$, $\left| \sum_{k \in pa_i}{\sign(\theta^{i,pa_i}_k) \theta^{j,pa_i}_k } \right| < \kappa$.
  \end{enumerate}
  Assumption (A-3) limits the magnitude of the effects that each node in the network receives from its parents. This is less restrictive than the neighborhood selection criterion, where the effects over all neighboring nodes are assumed to be bounded. In fact, empirical data indicate that the average number of upstream-regulators per gene in  regulatory networks is less than 2 \citep{leclerc2008ssr}. Thus, the size of parents of each node is small, but each hub node can affect many downstream nodes.

  Assumption (A-4) is referred to as \emph{neighborhood stability} and is equivalent to the \emph{irrepresentability} assumption proposed by \citet{huang2008als}. It has been shown that the lasso estimates are not in general variable selection consistent if this assumption is violated. \citet{huang2008als} considered adaptive lasso estimates with general initial weights and showed their variable selection consistency under a weaker form of irrepresentability assumption, referred to as \emph{adaptive irrespresentability}. We will show that when the initial weights for adaptive lasso are derived from the regular lasso estimates (as in (\ref{eqn_weight})), the assumption of neighborhood stability, as well as the less stringent assumption (A-3) are not required for establishing variable selection consistency of adaptive lasso. This relaxation in assumptions required for variable selection consistency, is a result of the consistency of regular lasso estimates, as well as the special structure of \textsc{dag}s. However, the results of this section can be extended to adaptive lasso estimates of the precision matrix, as well as regression models with fixed and random design matrices, under additional mild assumptions.

\subsection{Asymptotic Consistency of \textsc{dag} Estimation}\label{asymp_results}

  Our first result studies the variable selection consistency of the lasso penalty.
  \begin{theorem}[Variable Selection Consistency of Lasso]\label{Thm_Lasso}
    Suppose that (A-1)-(A-4) hold and $\lambda \asymp d n^{-(1-\zeta)/2}$ for some $b < \zeta < \xi$ and $d > 0$. Then there exist constants $c_{(i)}, \ldots, c_{(iv)} > 0$ for the lasso estimation problem, such that for all $i \in V$, as $n \rightarrow \infty$
  \begin{enumerate}
    \item(i)]
        Estimation of the direction of influence: $$\prob\left\{ \sign{(\hat{\theta}^{i,pa_i}_j )} = \sign{(\theta^{i,pa_i}_j ) \text{for all} j \in pa_i } \right\} = 1 - O\left\{ \exp{(-c_{(i)} n^\zeta)} \right\}$$
    \item[(ii)]
        Control of type I error: $\prob\left( \hat{pa}_i \subseteq pa_i \right) = 1 - O\left\{\exp{(-c_{(ii)} n^\zeta)}\right\}$.
    \item[(iii)]
        Control of type II error: $\prob\left( pa_i \subseteq \hat{pa}_i \right) = 1 - O\left\{ \exp{(-c_{(iii)} n^\zeta)} \right\}$.
    \item[(iv)]
        Let $\hat{E}$ be the lasso estimate for the set of edges in the network. Then $$\prob(\hat{E} = E) = 1 - O\left\{ \exp{(-c_{(iv)} n^\zeta)} \right\}.$$
    \end{enumerate}
  \end{theorem}
  \begin{proof}
    The proof of this theorem follows from arguments similar to those presented in \citet{meinshausen2006hdg}, with minor modifications and replacing conditional independence in undirected graphs with d-separation in \textsc{dag}s.
  \end{proof}
  The next result establishes similar properties for adaptive lasso estimates, without the assumptions of neighborhood stability. The proof of Theorem \ref{Thm_ALasso} makes use of consistency of a class of sparse estimates of the Cholesky factors of covariance matrices, established in Theorem 9 of \citet{lam2008src}. For completeness, we restate a simplified version of the theorem for our lasso problem, for which $\sigma_i = 1, \smallskip i = 1, \ldots p$ and the eigenvalues of the covariance matrix are bounded (see Remark \ref{remSigma}). Throughout this section, we denote by $s$ the total number of nonzero element of the true adjacency matrix, $A$ of the \textsc{dag}.
  \begin{theorem}[\citet{lam2008src}]\label{Thm_LassoConsistency}
    If $n^{-1}(s+1)\log{p}=o(1)$ and $\lambda = O\left\{{(\log{p} /n)}^{1/2}\right\}$, then $\| \hat{A} - A \|_F = O_p\left\{{(n^{-1}s \log{p})}^{1/2}\right\}$.
  \end{theorem}
  It can be seen from Theorem \ref{Thm_LassoConsistency} that lasso estimates are consistent as long as $n^{-1}(s+1)\log{p}=o(1)$. To take advantage of this result, we replace (A-0) with the following assumption
  \begin{enumerate}
    \item(A-$0^\prime$)]
        For some $a > 0$, $p = p(n) = O(n^a)$ as $n \rightarrow \infty$. Also, $\max_{i \in V} \card{(pa_i)} = O(n^b)$ as $n \rightarrow \infty$, where $s n^{2b-1}\log{n} = o(1)$ as $n \rightarrow \infty$.
  \end{enumerate}
  Assumption (A-$0^\prime$) further restricts the number of parents of each node and also enforces a restriction on the total number of nonzero elements of the adjacency matrix. Condition $s n^{2b-1} \log{n} = o(1)$, implies that $b < 1/2$. Therefore, although the consistency of adaptive lasso in Theorem \ref{Thm_ALasso} is established without making any further assumptions on the structure of the network (compared to Theorem \ref{Thm_Lasso}), it is achieved at the price of requiring higher degree of sparsity in the network.
  We now state the main result regarding variable selection consistency of adaptive lasso. Note that the theorem only requires assumptions (A-$0^\prime$), (A-1) and (A-2), and assumptions (A-3) and (A-4) are no longer required.
  \begin{theorem}[Variable Selection Consistency of Adaptive Lasso]\label{Thm_ALasso}
    Consider the adaptive lasso estimation problem, where the initial weights are calculated using regular lasso estimates of the adjacency matrix of the graph in (\ref{AOpt_DAG_sep2}). Suppose (A-$0^\prime$) and (A-1)-(A-2) hold and $\lambda \asymp d n^{-(1-\zeta)/2}$ for some $b < \zeta < \xi$ and $d > 0$. Also suppose that the initial lasso estimates are found using a penalty parameter $\lambda^0$ that satisfies $\lambda^0 = O\left\{{(\log{p}/n)}^{1/2}\right\}$. Then there exist constants $c_{(i)}, \ldots, c_{(iv)} > 0$ such that for all $i \in V$, as $n \rightarrow \infty$ (i)-(iv) in Theorem \ref{Thm_Lasso} hold.
  \end{theorem}
  \begin{proof}
    A proof is given in the Appendix.
  \end{proof}

\subsection{Choice of the Tuning Parameter}\label{tuning}
  Both lasso, as well as adaptive lasso estimates of the adjacency matrix, depend on the choice of the tuning parameter $\lambda$. Different methods have been proposed for selecting the value of tuning parameter, including cross validation \citep{RothEtal:08, fan2007nev} and Bayesian Information Criteria (\textsc{bic}) \citep{yuan2007msa}. In both of these methods, the value of $\lambda$ is chosen to minimize the (penalized) likelihood function. However, choices of $\lambda$ that result in the optimal classification error do not guarantee a small error for the network reconstruction. We propose next a choice of $\lambda$ for \textsc{dag}s. Consider the following choice of the tuning parameter for the general weighted lasso problem with weights $w_{ij}$. Let $Z^*_q$ denote the $(1-q)$th quantile of standard normal distribution, and define
  \begin{equation}\label{eqnLambda}
    \lambda_i(\alpha) = 2 n^{-1/2} Z^*_{\frac{\alpha}{2p(i-1)}}
  \end{equation}
  The following theorem, states that such a choice controls the probability of falsely joining two distinct ancestral sets, defined next.
  \begin{definition}
    For every node $i \in V$, the ancestral set of node $i$, $AN_i$ consists of all nodes $j$ such that $j$ is an ancestor of $i$ or $i$ is an ancestor of $j$ or $i$ and $j$ have a common ancestor $k$.
  \end{definition}
  \begin{theorem}[Choice of the Tuning Parameter]\label{Thm_tuning}
    Under the assumptions of Theorems \ref{Thm_Lasso} and \ref{Thm_ALasso} above (for lasso and adaptive lasso, respectively), for all $n \in \mathbb{N}$ the solution of the general weighted lasso estimation problem with tuning parameter determined in (\ref{eqnLambda}) satisfies
    \[
        \prob( \exists i \in V: \hat{AN}_i \nsubseteq AN_i ) \le \alpha
    \]
  \end{theorem}
  \begin{proof}
    A proof is given in the Appendix.
  \end{proof}
  Note that as in \citet{meinshausen2006hdg}, Theorem \ref{Thm_tuning} is true for all values of $p$ and $n$. However, the theorem does not provide any guarantee on false positive or false negative probabilities for individual edges in the graph. We also need to determine the optimal choice of penalty parameter $\lambda^0$ for the first phase of the adaptive lasso, where the weights are estimated using lasso. Since the goal of the first phase is to achieve prediction consistency, cross validation can be used to determine the optimal choice of $\lambda^0$. On the other hand, it is easy to see that the error-based proposal in (\ref{eqnLambda}) satisfies the requirement of Theorem \ref{Thm_LassoConsistency} and can therefore be used to define $\lambda^0$. It is however recommended to use a higher value of significance level in estimating the initial weights, in order to prevent an over-sparse solution.

%%%%%%%%%%%%%%%%%%%%%%%
\section{Performance Analysis}\label{PerfAnal}

\subsection{Preliminaries}\label{PerfAnal_prem}
  In this section, we consider examples of estimating \textsc{dag}s of varying number of edges from randomly generated data. To randomly generate data from \textsc{dag}s, one needs to generate lower-triangular adjacency matrices with sparse nonzero elements, $\rho_{ij}$. We use the random \textsc{dag} generator in the R-package \texttt{pcalg} (\cite{kalisch2007ehd}) which also controls the neighborhood size. The sparsity levels in \textsc{dag}s with different sizes are set according to the theoretical bounds in section \ref{asymptotics}, as well as the recommendations of \citep{kalisch2007ehd}, for neighborhood size. More specifically, in simulations throughout this section, we use a maximum neighborhood size of 5, while limiting the total number of true edges to be equal to the sample size $n$.

  Different measures of structural difference can be used to evaluate the performance of estimators. The Structural Hamming Distance (\textsc{shd}) between the structures of the estimated and true \textsc{dag}s, represents the number of edges that are not in common between the two graphs and is equal to the sum of false positive and false negative edges in the estimated graph. The main drawback of this measure is its dependency on the number of nodes, as well as the sparsity of the network. The second measure of goodness of estimation considered here is the Matthews Correlation Coefficient (\textsc{mcc}). \textsc{mcc} is commonly used to assess the performance of binary classification methods and is defined as
  \begin{equation}
    \textsc{mcc} = \frac{(\textsc{tp} \times \textsc{tn}) - (\textsc{fp} \times \textsc{fn})}{{(\textsc{tp} + \textsc{fp})(\textsc{tp} + \textsc{fn})(\textsc{tn} + \textsc{fp})(\textsc{tn} + \textsc{fn})}^{1/2} }
  \end{equation}
  \noindent where \textsc{tp}, \textsc{tn}, \textsc{fp} and \textsc{fn} correspond to true positive, true negative, false positive and false negative, respectively. The value of \textsc{mcc} ranges from $-1$ to $1$ with larger values corresponding to better fits ($-1$ and $1$ represent worst and best fits, respectively). Finally, in order to compare the performance of different estimation methods with theoretical bounds established in section \ref{asymp_results}, we also report the values of false positive rates.

  The performance of both the PC-Algorithm as well as our proposed estimators based on the choice of tuning parameter in (\ref{eqnLambda}) vary with different values of significance level $\alpha$. In the following experiments, we first investigate the appropriate choice of $\alpha$ for each estimator. We then compare the performance of the estimators with an optimal choice of lambda using both numerical measures of performance, as well as gray-scale images of the estimated \textsc{dag}s with the true structure. Gray-scale images are obtained by calculating the proportion of times that a specific edge is present in the simulations (i.e. $\hat{A}_{ij} \ne 0$). To offset the effect of numerical instability, we consider an edge present, if $|\hat{A}_{ij}| > 10^{-4}$.

\subsection{Estimation of \textsc{dag}s from Normally Distributed Observations}\label{SimNormal}

  We begin with an example that illustrates the differences between estimation of \textsc{dag}s and the conditional independence networks (\textsc{cig}). The first two images in Figure \ref{figSimDAGplot} represent a randomly generated \textsc{dag} of size 50 along with the gray-scale image of the average precision matrix estimated based on 100 observations using the graphical lasso algorithm (implemented in the R package \texttt{glasso}). To control the probability of falsely connecting two components of the graph, the value of the tuning parameter for \texttt{glasso} is defined based on the error-based proposal of \citet{banerjee2008mst} (Theorem 2). It can be seen that the \textsc{cig} has many more edges ($8\%$ false positives compared to $1\%$ for lasso and adaptive lasso), and does not reveal the true structure of the underlying \textsc{dag}. It can be seen that although methods of estimating \textsc{cig}s are computationally efficient, they should not be used in applications like estimation of gene regulatory networks, where the underlying graph is directed.
    \begin{figure}[h!]
        \centering
        \scalebox{0.5}
        {\includegraphics[clip=TRUE, trim=0cm 0cm 0cm 0cm]{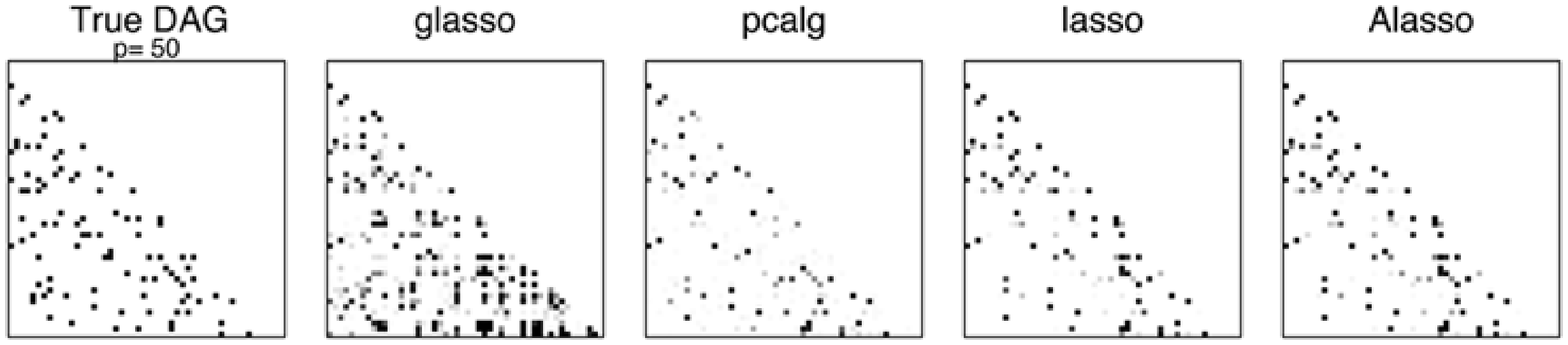}} %Crops off from left, bottom, right, top
        \caption{True \textsc{dag} along with estimates from Gaussian observations using \texttt{glasso}, \texttt{pcalg}, \texttt{lasso} and \texttt{Alasso}. Gray scale represents the percentage of inclusion of edges.} \label{figSimDAGplot}
    \end{figure}
%    \begin{figure}[t!]
%        \figurebox{2.3cm}{11.3cm}{}[Aest112209normal_All_Image.eps]
%        \caption{True \textsc{dag} along with estimates from Gaussian observations using \texttt{glasso}, \texttt{pcalg}, \texttt{lasso} and \texttt{Alasso}. Gray scale represents the percentage of inclusion of edges.} \label{figSimDAGplot}
%    \end{figure}
  In simulations throughout this section, the sample size is fixed at $n=100$, and estimators are evaluated for an increasing number of nodes ($p=50,100,200$). Figure \ref{figSimHamming_normal} shows the mean and standard deviation of Hamming distances for estimates based on the PC-Algorithm (\texttt{pcalg}), as well as the proposed lasso (\texttt{lasso}) and adaptive lasso (\texttt{Alasso}) methods for different values of the tuning parameter $\alpha$ and different network sizes. For all values of $p$ and $\alpha$, the adaptive lasso estimate gives the best results, and the proposed penalized likelihood methods outperform the PC-Algorithm. This difference becomes more significant as the size of the network increases.

  As mentioned in section \ref{DAGproblem}, it is not always possible to estimate the direction of the edges of a \textsc{dag} and therefore, the estimate from the PC-Algorithm may include undirected edges. Since our penalized likelihood methods assume knowledge of the ordering of variables and estimate the structure of the network, in the simulations considered here, we only estimate the skeleton (structure) of the network using the PC-Algorithm. We then then use the ordering of the variables to determine the direction of the edges. The performance of the the PC-Algorithm for estimation of the partially completed \textsc{dag} (\textsc{pcdag}) may therefore be worse than the results reported here.

  In the simulation results reported here, observations are generated according to the linear structural equation model (\ref{eqnSEMlin}) with standard normal latent variables and $\rho_{ij} = \rho = 0 \cdot 8$. Additional simulation studies with different values of $\sigma$ and $\rho$ indicate that changes in $\sigma$ do not have a considerable affect on the performance of the proposed models. On the other hand, as the magnitude of $\rho$ decreases, the performance of the proposed methods (as well as the \texttt{pcalg}) algorithm deteriorates, but the findings of the above comparison remain unchanged.

  The above simulation results suggest that the optimal performance of the PC-Algorithm is achieved when $\alpha = 0.01$. The performance of lasso and adaptive lasso methods is less sensitive to the choice of $\alpha$; however, $\alpha = 0.10$ seems to deliver more reliable estimates. Similar results were also observed in simulations with other choices of $\sigma$ and $\rho$. In addition, our extended simulations indicate that the performance of adaptive lasso does not vary significantly with the choice of power ($\gamma$), and therefore we only present the results for $\gamma = 1$.
    \begin{figure}[t!]
        \centering
        \scalebox{0.6}
        {\includegraphics[clip=TRUE, trim=0cm 0cm 0cm 0cm]{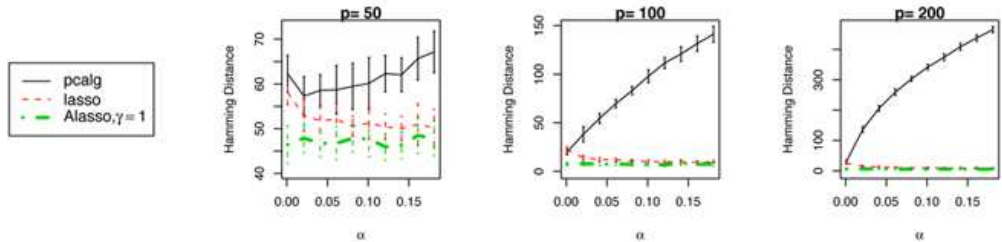}}%Crops off from left, bottom, right, top
        \caption{Hamming Distance for estimation of \textsc{dag} using \texttt{pcalg}, \texttt{lasso} and \texttt{Alasso} from normal observations.}
        \label{figSimHamming_normal}
    \end{figure}
%    \begin{figure}[t]
%        \figurebox{2.5cm}{11cm}{}[Aest112209normal_Hamming.eps]
%        \caption{Hamming Distance for estimation of \textsc{dag} using \texttt{pcalg}, \texttt{lasso} and \texttt{Alasso} from normal observations.}
%        \label{figSimHamming_normal}
%    \end{figure}
  Figure \ref{figSimDAGplot} represents images of estimated and true \textsc{dag}s created based on the above considerations for tuning parameters for $p=50$. Similar results are observed for $p=100, p=200$, and are excluded to conserve space. Plots in figure \ref{figSimAll_normal} compare the performance of the three methods with the optimal settings of tuning parameters, over a range of values of $p$. It can be seen that the values of \textsc{mcc} confirm the above findings based on \textsc{shd}. However, false positive and true positive rates only focus on one aspect of estimation at a time and do not provide a clear distinction between the methods. These plots also suggest that estimates from the \texttt{pcalg} may vary significantly. This variation is represented more significantly in terms of false positive rates, where standard deviations for estimates based are up to 10 times larger than those of \texttt{lasso} and \texttt{Alasso} estimates ($\sim 40\%$ for \texttt{pcalg} compared to $< 4\%$ for \texttt{lasso} and \texttt{Alasso}).
    \begin{figure}[t!]
        \centering
        \scalebox{0.57}
        {\includegraphics[clip=TRUE, trim=0cm 0cm 0cm 0cm]{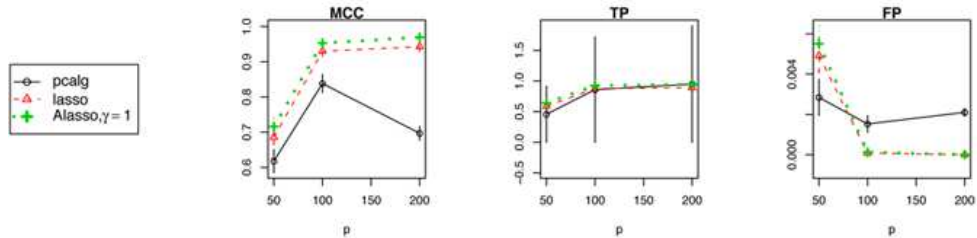}}%Crops off from left, bottom, right, top
        \caption{\textsc{mcc}, \textsc{fp} and \textsc{tp} for estimation of \textsc{dag} using \texttt{pcalg}, \texttt{lasso} and \texttt{Alasso} from normal observations.}
        \label{figSimAll_normal}
    \end{figure}
%    \begin{figure}[h!]
%        \figurebox{2.5cm}{11cm}{}[Aest112209normal_All_OtherPlots.eps]
%        \caption{\textsc{mcc}, \textsc{fp} and \textsc{tp} for estimation of \textsc{dag} using \texttt{pcalg}, \texttt{lasso} and \texttt{Alasso} from normal observations.}
%        \label{figSimAll_normal}
%    \end{figure}

  As mentioned in section \ref{DAGproblem}, the representation of conditional independence in \textsc{dag}s adapted in our proposed algorithm, is not restricted to normally distributed random variables. Moreover, if the underlying structural equations are linear, the method proposed in this paper can correctly estimate the underlying \textsc{dag}. In order to assess the sensitivity of the estimates to the underlying distribution, we performed two simulation studies with non-Normal observations. In both simulations, observations were generated according to a linear structural model. However, in the first simulation, each latent variable was generated from a mixture of a standard normal and a t-distribution with 3 degrees of freedom, while in the second simulation, a t-distribution with 4 degrees of freedom was used. The performance of the proposed algorithm for non-normal observations was similar to the case of Gaussian observations, with \texttt{Alasso} providing the best estimates, and the performance of penalized methods improving as the dimension and sparsity increase.

\subsection{Sensitivity to Perturbations in the Ordering of the Variables}\label{SimPermute}

  Algorithm \ref{algDAG} assumes a known ordering of the variables. The superior performance of the proposed penalized likelihood methods in comparison to the PC-Algorithm may be explained by the fact that additional information about the order of the variables significantly simplifies the problem of estimating \textsc{dag}s. Therefore, when such additional information is available, estimates using the PC-Algorithm suffer from a natural disadvantage. However, as the underlying network becomes more sparse, the network includes fewer complex structures and it is expected that the ordering of variables should play a less significant role.

  Next, we study the performance of the proposed penalized likelihood methods as well as the PC-Algorithm in problems where the ordering of variables is unknown. To this end, we generate normally distributed observations from the latent variable model of section \ref{latent}. We then randomly permute the order of variables in the observation matrix and use the permuted matrix to estimate the original \textsc{dag}. Figure \ref{figSimAll_permute} illustrates the performance of the three methods for choices of $\alpha$ described in section \ref{SimNormal}. It can be seen that for small, dense networks, the PC-Algorithm outperforms the proposed methods. This is expected since the change in the order of variables causes the algorithm to include unnecessary moral edges, while failing to recognize some of the existing associations. On the other hand, as the size of the network and correspondingly the degree of sparsity in the network increase, the local structures become simpler and therefore the ordering of the variables becomes less crucial. Thus, the performance of penalized likelihood algorithms is improved compared to that of the PC-Algorithm. For the high dimensional sparse case, where the computational cost of the PC-Algorithm becomes more significant, both lasso and adaptive lasso methods outperform the PC-Algorithm.
    \begin{figure}[h!]
        \centering
        \scalebox{0.57}
        {\includegraphics[clip=TRUE, trim=0cm 0cm 0cm 0cm]{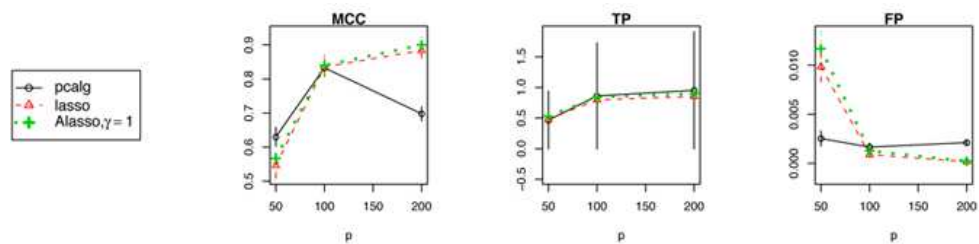}}%Crops off from left, bottom, right, top
        \caption{\textsc{mcc}, \textsc{fp} and \textsc{tp} for estimation of \textsc{dag} using \texttt{pcalg}, \texttt{lasso} and \texttt{Alasso} with random ordering.}
        \label{figSimAll_permute}
    \end{figure}
%    \begin{figure}[h!]
%        \figurebox{2.5cm}{11cm}{}[Aest112209normalPermute_All_OtherPlots.eps]
%        \caption{\textsc{mcc}, \textsc{fp} and \textsc{tp} for estimation of \textsc{dag} using \texttt{pcalg}, \texttt{lasso} and \texttt{Alasso} with random ordering.}
%        \label{figSimAll_permute}
%    \end{figure}

\section{Real Data Application}\label{RealData}

\subsection{Analysis of Cell Signalling Pathway Data}\label{ex1}
  \citet{sachs2003cps} carried out a set of flow cytometry experiments on signaling networks of human immune system cells. The ordering of the connections between pathway components were established based on perturbations in cells using molecular interventions and we consider the ordering to be known a priori. The data set includes $p=11$ proteins and $n=7466$ samples.

  \citet{friedman2008sic} analyzed this data set using the \texttt{glasso} algorithm. They estimated the graph for a range of values of the $\ell_1$ penalty and reported moderate agreement (around $50\%$ false positive and false negative rates) between one of the estimates and the findings of \citet{sachs2003cps}. True and estimated signaling networks using the PC-Algorithm, as well as both lasso and adaptive lasso algorithms, along with performance measures are given in Figure \ref{figCellSignal}. The estimated network using the \texttt{pcalg} includes a number of undirected edges. As in the simulation studies, we only estimate the structure of the network using the \texttt{pcalg} and determine the direction of edges by enforcing the ordering of nodes in the \textsc{dag}. It can be seen that the adaptive lasso and lasso provides estimates that are closest to the true structure.
    \begin{figure}[t!]
        \centering
        \scalebox{0.75}
        {\includegraphics[clip=TRUE, trim=0cm 0cm 0cm 0cm]{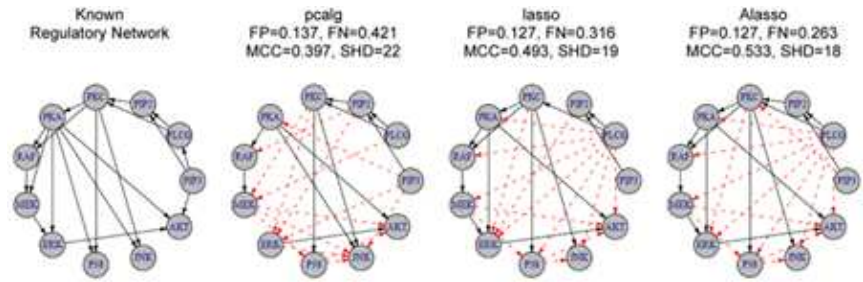}} %Crops off from left, bottom, right, top
        \caption{Known and estimated networks for human cell signalling data. True edges (True Positives in estimated networks) are marked with solid blue arrows, while False Positives are indicated by dashed red arrows.} \label{figCellSignal}
    \end{figure}
%    \begin{figure}[t!]
%        \figurebox{3cm}{10.5cm}{}[SachsPlots112209.eps]
%        \caption{Known and estimated networks for human cell signalling data. True positives are marked with solid arrows, while false positives are indicated by dashed arrows.}
%        \label{figCellSignal}
%    \end{figure}

\subsection{Transcription Regulatory Network of E-coli}\label{ex2}
  Transcriptional regulatory networks play an important role in controlling the gene expression in cells and incorporating the underlying regulatory network results in more efficient estimation and inference \citep{shojaie2009NetBasedGSA, shojaie2009NetEnrich}. \citet{kao2004tbd} proposed to use Network Component Analysis to infer transcriptional regulatory network of \emph{Escherichia coli} (E-coli). They also provide whole genome expression data over time ($n=24$), as well as information about the known regulatory network of E-coli.

  In this application, the set of transcription factors (TFs) are known \emph{a priori} and the goal is to find connections among transcription factors and regulated genes through analysis of whole genome transcriptomic data. Therefore, the algorithm proposed in this paper can be used by exploiting the natural hierarchy of TFs and genes. \citet{kao2004tbd} provide gene expression data for 7 transcription factors and 40 genes regulated by these TFs. Figure \ref{figEcoli} presents the known regulatory network of E-coli along with the networks estimated using three different methods as well as different measures of performance. The relatively poor performance of the algorithms in this example can be partially attributed to the small sample size. However, it is also known that no single source of transcriptomic data is expected to successfully reveal the regulatory networks and methods that combine different sources of data are considered to be more efficient. It can be seen that the PC-Algorithm can only detect one of the true regulatory connections, while the proposed algorithm, with both lasso, as well as adaptive lasso penalties, offers significant improvements over the results of the PC-Algorithm, mostly due to the significant drop in proportion of false negatives (from $97\%$ for the PC-Algorithm to $63\%$ for adaptive lasso). Although the estimates based on lasso and adaptive lasso penalties are very similar, the choice of the best estimate depends on the performance evaluation metric.
    \begin{figure}[t!]
        \centering
        \scalebox{0.75}
        {\includegraphics[clip=TRUE, trim=0cm 0cm 0cm 0cm ]{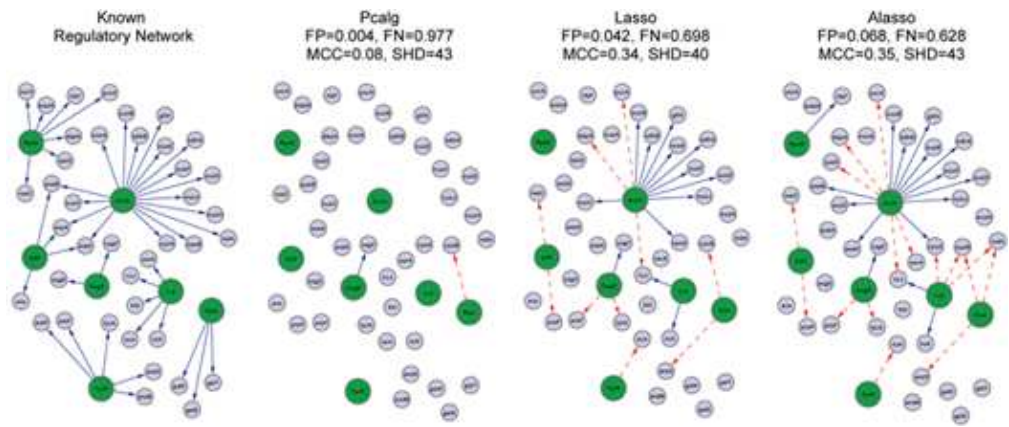}} %Crops off from left, bottom, right, top
        \caption{Known and estimated transcription regulatory network of E-coli. Large nodes indicate the transcription factors (\textsc{TF}s) and smaller nodes refer to their regulated genes (ORFs). True edges (True Positives in estimated networks) are marked with solid blue arrows, while False Positives are indicated by dashed red arrows.} \label{figEcoli}
    \end{figure}
%    \begin{figure}[t!]
%        \figurebox{5.2cm}{12.5cm}{}[EcoliPlots112209.eps]
%        \caption{Known and estimated transcription regulatory network of E-coli. Large nodes indicate the transcription factors (TFs) and smaller nodes refer to their regulated genes. True positives are marked with solid arrows, while false positives are indicated by dashed arrows.}
%        \label{figEcoli}
%    \end{figure}

%%%%%%%%%%%%%%%%%%%%%%%
\section{Conclusion}\label{conclusion}
  We proposed efficient penalized likelihood methods for estimation of the structure of \textsc{dag}s when variables inherit a natural ordering. Both lasso and adaptive lasso penalties were considered in this paper. However, the proposed algorithm can also be used for estimation of adjacency matrix of \textsc{dag}s under other choices of penalty, as long as the penalty function is applied to individual elements of the adjacency matrix. %(e.g. the \textsc{scad} penalty).

  There are a number of biological applications where the ordering of the variables is known \emph{a priori}. Estimation of transcriptional regulatory networks from gene expression data and reconstruction of causal networks from temporal observations, based on the concept of Granger causality \citep{granger1969icr} are areas of potential application for the proposed algorithm. Our simulation studies also indicate that the correct ordering of variables is less crucial for estimation high dimensional sparse \textsc{dag}s. Thus, even in high dimensional sparse applications where the ordering among variables is not known, the methods proposed in this paper may be an efficient alternative for search based methods of estimating \textsc{dag}s.
%%%%%%%%%%%%%%%%%%%%%%%%%%%%%%%%%%%%%%%%%%
%\section*{Acknowledgement}
%The work of George Michailidis was partially supported by NIH grant 5P 41RR018627 and MEDC grant GR-687. The authors would like to thank the authors of \citet{kalisch2007ehd} and \citet{friedman2008rpg} for making the R-packages \texttt{pcalg} and \texttt{glmnet} available. We are especially thankful to Markus Kalisch and Trevor Hastie for help with technical difficulties with these packages. We would like to thank the editor, an associate editor as well as a referee, whose comments and suggestions significantly improved the clarity of the manuscript.

%%%%%%%%%%%%%%%%%%%%%%%%%%%%%%%%%%%%%%%%%%

%%%%%%%%%%%%%%%%%%%%%%%%%%%
%\newpage
%\appendix
\section*{Appendix: Technical Proofs}
  Recall from section \ref{asymp_assump} that $\theta^{i} = A_{i,\underline{i-1}}$ denotes the entries of the $i$th row of the adjacency matrix to the left of the diagonal. Also, denote by $\theta^{i,\mathcal{I}}$ the estimate for the $i$th row with values outside the set of indices $\mathcal{I}$ set to zero and let $\theta^{i,\mathcal{I}}_j$ be the $j$ component of $\theta^{i,\mathcal{I}}$.

  The following lemma is a consequence of the Karush–Kuhn–Tucker conditions for the general weighted lasso problem and is used in the proof of Theorems \ref{Thm_ALasso} and \ref{Thm_tuning}.
  \begin{lemma}\label{lemma_AlassoKKT}
    Let $\hat{\theta}^{i,\mathcal{I}}$ be the general weighted lasso estimate of $\theta^{i,\mathcal{I}}$, i.e.
        \begin{equation}\label{eqnThetaHat}
            \hat{\theta}^{i,\mathcal{I}} = \argmin_{\theta:\theta_k = 0 \forall k \notin \mathcal{I}}  { \left\{
                                                 n^{-1} \| \mathcal{X}_i - \mathcal{X} \theta \|_2^2 + \lambda \sum_{k=1}^{p}{|\theta_k| w_{ik}}
                                            \right\} }
        \end{equation}
    \noindent Let
        \begin{equation*}
            G_j(\theta) = -2n^{-1} \T{\mathcal{X}}_j (\mathcal{X}_i - \mathcal{X} \theta)
        \end{equation*}
    \noindent and $w_i$ be the vector of initial weights in adaptive lasso estimation problem. Then a vector $\hat{\theta}$ with $\hat{\theta}_k = 0, \forall k \notin \mathcal{I}$ is a solution of (\ref{eqnThetaHat}) iff $\forall j \in \mathcal{I}, G_j(\theta) = -\sign{(\hat{\theta}_j)} w_{ij} \lambda$ if $\hat{\theta}_j \neq 0$ and $| G_j(\theta) | \le w_{ij} \lambda$ if $\hat{\theta}_j = 0$. Moreover, if the solution is not unique and $| G_j(\theta) | < w_{ij} \lambda$ for some solution $\hat{\theta}$, then $\hat{\theta}_j = 0$ for all solutions of (\ref{eqnThetaHat}).
  \end{lemma}
  \begin{proof}
    The proof of the lemma is identical to the proof of Lemma (A.1) in \citet{meinshausen2006hdg} (except for inclusion of general weights $w_{ij}$) and is therefore omitted.
  \end{proof}

%%%%%%%%%%%%%%%%%%%%%%%%%%%%%%%%
\begin{proof}[of Theorem \ref{Thm_ALasso}]
    %%%%%%%%%%%%%%%%
    %%  part (i)
    %%%%%%%%%%%%%%%%
    To prove (i), note that by Bonferroni's inequality, and the fact that $\card{(pa_i)} = o(n)$ as $n \rightarrow \infty$, it suffices to show that there exists some $c_{(i)} > 0$ such that for all $i \in V$ and for every $j \in pa_i$,
    \[
        \prob\left\{ \sign{(\hat{\theta}^{i,pa_i}_j )} = \sign{(\theta^{i,pa_i}_j )} \right\} = 1 - O\left\{\exp{( -c_{(i)} n^\zeta )}\right\} \hspace{0.5cm} \text{ as } n \rightarrow \infty
    \]
    Let $\hat{\theta}^{i,pa_i}(\beta)$ be the estimate of $\theta^{i,pa_i}$ in (\ref{eqnThetaHat}), with the $j$th component fixed at a constant value $\beta$,
    \begin{equation}\label{eqnThetaHat_beta}
        \hat{\theta}^{i,pa_i}(\beta) = \argmin_{\theta \in \Theta_\beta} { \left\{
                                                n^{-1} \| \mathcal{X}_i - \mathcal{X} \theta \|_2^2 + \lambda \sum_{k=1}^{p}{|\theta_k| w_k}
                                            \right\} }
    \end{equation}
    \noindent where $ \Theta_\beta \equiv \left\{ \theta \in \mathbb{R}^p: \theta_j = \beta, \theta_k = 0, \forall k \notin pa_i \right\} $.
    Note that for $\beta = \hat{\theta}^{i,pa_i}_j$, $\hat{\theta}^{i,pa_i}(\beta)$ is identical to $\hat{\theta}^{i,pa_i}$. Thus, if $\sign{(\hat{\theta}^{i,pa_i}_j )} \neq \sign{(\theta^i_j )}$, there would exist some $\beta$ with $\sign{(\beta)}\sign{(\theta^i_j)} \le 0$ such that $\hat{\theta}^{i,pa_i}(\beta)$ is a solution to (\ref{eqnThetaHat_beta}). Since $\theta^i_j \neq 0, \forall j \in pa_i$, it suffices to show that for all $\beta$ with $\sign{(\beta)}\sign{(\theta^i_j)} < 0$, with high probability, $\hat{\theta}^{i,pa_i}(\beta)$ can not be a solution to (\ref{eqnThetaHat_beta}).
    Without loss of generality, we consider the case where $\theta^i_j > 0$ ($\theta^i_j < 0$ can be shown similarly). Then if $\beta \le 0$, from Lemma \ref{lemma_AlassoKKT}, $\hat{\theta}^{i,pa_i}(\beta)$ can be a solution to (\ref{eqnThetaHat_beta}) only if $G_j(\hat{\theta}^i(\beta)) \ge -\lambda w_{ij}$. Hence, it suffices to show that for some $c_{(i)} > 0$ and all $j \in pa_i$ with $\theta^i_j > 0$,
    \begin{equation}\label{eqnThetaHat_beta2}
        \prob\left[ \sup_{\beta \le 0}{ \{ G_j\left( \hat{\theta}^i(\beta) \right) < - \lambda w_{ij} \} } \right] = 1 - O\left\{ \exp{(-c_{(i)}n^\zeta})\right\} \hspace{0.5cm} \text{ as } n \rightarrow \infty
    \end{equation}
    Define,
    \begin{equation}\label{eqnResid}
        \mathcal{R}_i(\beta) = \mathcal{X}_i - \mathcal{X} \hat{\theta}^i(\beta)
    \end{equation}
    \noindent For every $j \in pa_i$ we can write,
    \begin{equation}\label{eqnXj}
        X_j = \sum_{ k \in pa_i \backslash \{j\} }{ \theta^{ j , pa_i \backslash \{j\} }_k X_k } + Z_j
    \end{equation}
    where $Z_j$ is independent of $ \{ X_k; k \in pa_i \backslash \{j\} \} $. Then by (\ref{eqnXj}),
    \[
        G_j\left(\hat{\theta}^i(\beta)\right) = -2n^{-1} \T{\mathcal{Z}}_j \mathcal{R}_i(\beta) -
                \sum_{ k \in pa_i \backslash \{j\} }{ \theta^{ j , pa_i \backslash \{j\} } 2n^{-1} \T{\mathcal{X}_k} \mathcal{R}_i(\beta) }
    \]
    By Lemma \ref{lemma_AlassoKKT}, it follows that for all $k \in pa_i \backslash \{j\}$, $| G_k\left(\hat{\theta}^i(\beta)\right) | =     | 2n^{-1} \T{\mathcal{X}_k} \mathcal{R}_i(\beta) | \le \lambda w_{ik}$, thus,
    \begin{equation}\label{eqnG2}
        G_j\left(\hat{\theta}^i(\beta)\right) \le -2n^{-1} \T{\mathcal{Z}}_j \mathcal{R}_i(\beta) + \lambda
                \sum_{ k \in pa_i \backslash \{j\} }{ |\theta^{ j , pa_i \backslash \{j\} }| w_{ik} }
    \end{equation}
    \noindent Using the fact that $|\theta^{ j , pa_i \backslash \{j\} }| \le 1$, it suffices to show that
    \begin{equation}\label{eqnG3}
        \prob\left[ \sup_{\beta \le 0}{ \{ -2n^{-1} \T{\mathcal{Z}}_j\mathcal{R}_i(\beta) \} } <
                            -\lambda \sum_{k \in pa_i}{w_{ik}} \right] =
        1 - O\left\{\exp{(-c_{(i)}n^\zeta)}\right\} \hspace{0.5cm} \text{ as } n \rightarrow \infty,
    \end{equation}
    or equivalently,
    \begin{equation}\label{eqnG4}
        \prob\left[ \inf_{\beta \le 0}{ \{ 2n^{-1} \T{\mathcal{Z}}_j\mathcal{R}_i(\beta) \} } < \lambda \sum_{k \in pa_i}{w_{ik}} \right] = O\left\{\exp{(-c_{(i)}n^\zeta)}\right\} \hspace{0.5cm} \text{ as } n \rightarrow \infty.
    \end{equation}
    It is shown in Lemma A.2. of \citet{meinshausen2006hdg} that for any $q > 0$, there exists $c_{(i)} > 0$ such that for all $j \in pa_i$ with $\theta^i_j > 0$
    \begin{equation}\label{eqnG5}
        \prob\left[ \inf_{\beta \le 0}{ \{ 2n^{-1} \T{\mathcal{Z}}_j\mathcal{R}_i(\beta) \} } \le q\lambda \right] = O\left\{\exp{(-c_{(i)}n^\zeta)}\right\} \hspace{0.5cm} \text{ as } n \rightarrow \infty.
    \end{equation}
    However, by definition $w_{ik} \le 1$ and therefore, $\sum_{k \in pa_i}{w_{ik}} \le \card{(pa_i)} \le 1$, which implies that (i) follows from \ref{eqnG5}.

%
%    However, $\prob\left[ \inf_{\beta \le 0}{ \{ 2n^{-1} \T{\mathcal{Z}}_j\mathcal{R}_i(\beta) \} } < \lambda \sum_{k \in pa_i}{w_{ik}} \right]$ is less than or equal to $$\prob\left( \inf_{\beta \le 0}{ \{ 2n^{-1} \T{\mathcal{Z}}_j\mathcal{R}_i(\beta) \} } \le q\lambda, \sum_{k \in pa_i}{w_{ik}} \ge q\lambda \right)$$ for some $q > 0$, which is dominated by $\prob\left[ \inf_{\beta \le 0}{ \{ 2n^{-1} \T{\mathcal{Z}}_j\mathcal{R}_i(\beta) \} } \le q\lambda \right]$.
%    Therefore, it suffices to show that for any $q > 0$, there exists $c_{(i)} > 0$ such that for all $j \in pa_i$ with $\theta^i_j > 0$
%    \begin{equation}\label{eqnG5}
%        \prob\left[ \inf_{\beta \le 0}{ \{ 2n^{-1} \T{\mathcal{Z}}_j\mathcal{R}_i(\beta) \} } \le q\lambda \right] = O\left\{\exp{(-c_{(i)}n^\zeta)}\right\} \hspace{0.5cm} \text{ as } n \rightarrow \infty.
%    \end{equation}
%    The claim in (\ref{eqnG5}) is proved in Lemma A.2. of \citet{meinshausen2006hdg} which completes the proof of (i).
    %%%%%%%%%%%%%%%%
    %%  part (ii)
    %%%%%%%%%%%%%%%%
    To prove (ii), note that the event $\hat{pa}_i \nsubseteq pa_i$ is equivalent to the event that there exists a node $j \in \underline{i-1} \backslash pa_i$ such that $\hat{\theta}^i_j \neq 0$, i.e.
    \begin{equation}\label{eqn_ii1}
        \prob\left( \hat{pa}_i \subseteq pa_i \right) =
                1 - \prob\left( \exists j \in \underline{i-1} \backslash pa_i: \hat{\theta}^i_j \neq 0\right)
    \end{equation}
    By Lemma \ref{lemma_AlassoKKT}, and using the fact that by definition $w_{ij} \ge 1$
    \begin{eqnarray*}
      \prob\left( \exists j \in \underline{i-1} \backslash pa_i: \hat{\theta}^i_j \neq 0\right)
      & = &
        \prob\left( \exists j \in \underline{i-1} \backslash pa_i: | G_j(\hat{\theta}^{i,pa_i}) | \ge w_{ij} \lambda \right) \\
      & \le &
        \prob\left( \exists j \in \underline{i-1} \backslash pa_i: | G_j(\hat{\theta}^{i,pa_i}) | \ge q \lambda
          \text{ and } w_{ij} \lambda \le q \lambda \text{  for some } q \ge 1 \right) \\
      & \le &
        \prob\left( \exists j \in \underline{i-1} \backslash pa_i: w_{ij} \le q \text{  for some } q \ge 1 \right)
    \end{eqnarray*}
    Since $w_{ij} = 1 \vee |\tilde{\theta}^i_j|^{-\gamma}$, with $\tilde{\theta}^i_j$ the lasso estimate of the adjacency matrix from (\ref{AOpt_DAG_lasso}), using Lemma \ref{lemma_AlassoKKT},
    \begin{eqnarray*}
      \prob\left( \exists j \in \underline{i-1} \backslash pa_i: w_{ij} \le q \text{  for some } q>0 \right)
      & = &
      \prob\left( \exists j \in \underline{i-1} \backslash pa_i: |\tilde{\theta}^i_j| \ge q^{-1/\gamma} \text{  for some } q \ge 1 \right) \\
      & \le &
      \prob\left( \exists j \in \underline{i-1} \backslash pa_i: |\tilde{\theta}^i_j| \ge q' \text{  for some } q'>0 \right) \\
      & \le &
      \prob\left( \exists j \in \underline{i-1} \backslash pa_i: \tilde{\theta}^i_j \neq 0 \right) \\
      & = &
      \prob\left( \exists j \in \underline{i-1} \backslash pa_i: | G_j(\tilde{\theta}^{i,pa_i}) | \ge \lambda^0 \right)
    \end{eqnarray*}
    Since $\card{(pa_i)} = o(n)$, we can assume without loss of generality that $\card{(pa_i)} < n$, which implies that $\tilde{\theta}^{i,pa_i}$ is an almost sure unique solution to (\ref{eqnThetaHat}) with $\mathcal{I} = pa_i$. Let
    $$
        \mathcal{E} = \left\{ \max_{j \in \underline{i-1} \backslash pa_i}{| G_j(\tilde{\theta}^{i,pa_i}) |} < \lambda^0 \right\}.
     $$
    Then conditional on event $\mathcal{E}$, it follows from the first part of Lemma \ref{lemma_AlassoKKT} that $\tilde{\theta}^{i,pa_i}$ is also a solution of the unrestricted weighted lasso problem (\ref{eqnThetaHat}) with $\mathcal{I} = \underline{i-1}$. Since $\tilde{\theta}^{i,pa_i}_j = 0, \hspace{0.05cm} \forall j \in \underline{i-1} \backslash pa_i$, it follows from the second part of Lemma \ref{lemma_AlassoKKT} that $\tilde{\theta}^{i}_j = 0, \hspace{0.05cm} \forall j \in \underline{i-1} \backslash pa_i$. Hence
    \begin{eqnarray}\label{eqn_ii2}
        \prob\left( \exists j \in \underline{i-1} \backslash pa_i: \tilde{\theta}^i_j \neq 0 \right)
                    \le
                    1- \prob(\mathcal{E}) =
                    \prob\left( \max_{j \in \underline{i-1} \backslash pa_i}{|G_j(\tilde{\theta}^{i,pa_i})|} \ge \lambda^0 \right)
    \end{eqnarray}
    \noindent where,
    \begin{equation}\label{eqn_ii3}
        G_j(\tilde{\theta}^{i,pa_i}) = -2n^{-1} \T{\mathcal{X}}_j (\mathcal{X}_i - \mathcal{X}\tilde{\theta}^{i,pa_i})
    \end{equation}
    Since $\card{(V)} = O(n^a)$ for some $a > 0$, Bonferroni's inequality implies that to verify (\ref{eqn_ii1}) it suffices to show that there exists a constant $c_{(ii)} > 0$ such that for all $j \in \underline{i-1} \backslash pa_i$,
    \begin{equation}\label{eqn_ii4}
        \prob\left( |G_j(\tilde{\theta}^{i,pa_i}) | \ge \lambda^0 \right) =
                O\left\{\exp{(-c_{(ii)}n^\zeta)}\right\} \thickspace \text{ as } n \rightarrow \infty.
    \end{equation}
%    For $j \in \underline{i-1} \backslash pa_i$ one can write $X_j = \sum_{l \in pa_i}{\theta^{j,pa_i}_l X_l + R_j}$,
%    \begin{equation*}
%        X_j = \sum_{l \in pa_i}{\theta^{j,pa_i}_l X_l + R_j}
%    \end{equation*}
    where $R_j \sim N(0,\sigma^2_j), \thickspace \sigma^2_j \le 1$ and $R_j$ is independent from $X_l, l \in pa_i$. Similarly, with $R_i$ satisfying the same requirements as $R_j$, we get $X_i = \sum_{k \in pa_i}{\theta^{i,pa_i}_k X_k + R_i}$.
%    \begin{equation*}
%        X_i = \sum_{k \in pa_i}{\theta^{i,pa_i}_k X_k + R_i}
%    \end{equation*}

    Denote by $\mathcal{X}_{pa_i}$ the columns of $\mathcal{X}$ corresponding to $pa_i$ and let $\theta_{pa_i}$ the column vector of coefficients with dimension $\card{(pa_i)}$ corresponding to $pa_i$. Then,
    \begin{eqnarray*}\label{eqn_ii5}
        \prob\left\{ |G_j(\tilde{\theta}^{i,pa_i}) | \ge \lambda^0 \right\}
        & = &
        \prob\left\{ |-2n^{-1} \T{\mathcal{X}}_j(\mathcal{X}_i - \mathcal{X} \tilde{\theta}^{i,pa_i}) | \ge \lambda^0 \right\} \\
        & = &
        \prob\left[ |-2n^{-1} \T{ \{ \mathcal{X}_{pa_i} \theta^{j,pa_i}_{pa_i} + \mathcal{R}_j \} }
            \{ \mathcal{X}_{pa_i}( \theta^{i,pa_i}_{pa_i} - \tilde{\theta}^{i,pa_i}_{pa_i} ) + \mathcal{R}_i \} | \ge \lambda^0 \right].
    \end{eqnarray*}
    Therefore,
    \begin{eqnarray*}\label{eqn_ii6}
        \prob\left\{ |G_j(\tilde{\theta}^{i,pa_i}) | \ge \lambda^0 \right\}
        & \le &
        \prob\left\{ |-2n^{-1} \T{( \theta^{i,pa_i}_{pa_i} - \tilde{\theta}^{i,pa_i}_{pa_i} )} \T{\mathcal{X}}_{pa_i} \mathcal{X}_{pa_i} \theta^{j,pa_i}_{pa_i} | \ge \lambda^0 / 3 \right\} \\
        & + &
        \prob\left\{ |-2n^{-1} \T{( \theta^{i,pa_i}_{pa_i} - \tilde{\theta}^{i,pa_i}_{pa_i} )} \T{\mathcal{X}}_{pa_i}\mathcal{R}_j | \ge \lambda^0 / 3 \right\} \\
        & + &
        \prob\left\{ |-2n^{-1} \T{( \mathcal{X}_{pa_i} \theta^{j,pa_i}_{pa_i} + \mathcal{R}_j)} \mathcal{R}_i | \ge \lambda^0 / 3 \right\} \\
        & \equiv &
        \textrm{I} + \textrm{II} + \textrm{III}
    \end{eqnarray*}
    Let $\mathbf{1}_{pa_i}$ denote a vector of 1's of dimension $\card{(pa_i)}$. Then, using the fact that $|\theta^{j,pa_i}_{l}| \le 1, \text{for all } l \in pa_i$, we can write $\textrm{I} \le \prob\left(
                        2 \| \theta^{i,pa_i}_{pa_i} - \tilde{\theta}^{i,pa_i}_{pa_i} \|_{\infty}
                        n^{-1} \T{(\mathcal{X}_{pa_i}\mathbf{1}_{pa_i})} \mathcal{X}_{pa_i}\mathbf{1}_{pa_i} \ge \lambda^0 / 3
                    \right) $.
    Then $\T{\mathcal{X}}_{pa_i}\mathcal{X}_{pa_i} \sim \mathbb{W}_{\card{(pa_i)}}(\Sigma_{pa_i},n)$ where $\mathbb{W}_{m}(\Sigma,n)$ denotes a Wishart distribution with mean $n\Sigma$. Hence, from properties of the Wishart distribution, we get $\T{(\mathcal{X}_{pa_i}\mathbf{1}_{pa_i})} \mathcal{X}_{pa_i}\mathbf{1}_{pa_i} \sim \mathbb{W}_{1}(\T{\mathbf{1}}_{pa_i}\Sigma_{pa_i}\mathbf{1}_{pa_i},n)$.

    Since $pa_i$ also forms a \textsc{dag}, the eigenvalues $\Sigma_{pa_i}$ are bounded (see Remark \ref{remSigma}), and hence
    \begin{equation}\label{eqn_ii7}
    \T{\mathbf{1}}_{pa_i}\Sigma_{pa_i}\mathbf{1}_{pa_i} \le \card{(pa_i)} \phi_{\max}(\Sigma_{pa_i})
    \end{equation}
    Therefore, if $Z \sim \chi^2_1$, $n^{-1} \T{(\mathcal{X}_{pa_i}\mathbf{1}_{pa_i})} \mathcal{X}_{pa_i}\mathbf{1}_{pa_i}$ is stochastically smaller than $\card{(pa_i)}\phi_{\max}(\Sigma_{pa_i})Z$. On the other hand, by Theorem \ref{Thm_LassoConsistency},
    \begin{equation*}
        \|A - \tilde{A} \|_F = O_p \left\{ {(n^{-1} s \log{p})}^{1/2} \right\}
    \end{equation*}
    and hence
    \begin{equation}\label{eqn_ii8}
        \| \theta^{i,pa_i}_{pa_i} - \tilde{\theta}^{i,pa_i}_{pa_i} \|_{\infty} = O_p \left\{ {(n^{-1} s \log{p})}^{1/2} \right\}
    \end{equation}
    Noting that $\card{(pa_i)} = O(n^b), \thickspace b < 1/2 $ and $p = O(n^a), \thickspace a > 0$, (\ref{eqn_ii7}) and (\ref{eqn_ii8}) imply that
    \begin{equation*}
        \| \theta^{i,pa_i}_{pa_i} - \tilde{\theta}^{i,pa_i}_{pa_i} \|_{\infty} \card{(pa_i)} \phi_{\max}(\Sigma_{pa_i})
        = O_p \left\{ {(s n^{2b-1} a \log{n} )}^{1/2} \right\}
    \end{equation*}
    By (A-$0^\prime$), $s n^{2b-1} \log{n} = o(1)$ and hence by Slutsky's Theorem and properties of the $\chi^2$-distribution, there exists $c_{(\textrm{I})} > 0$ such that for all $j \in \underline{i-1} \backslash pa_i$,
    \begin{equation*}
        \textrm{I} =  O\left\{ \exp{(-c_{(\textrm{I})} n^\zeta)} \right\} \thickspace \text{ as } n \rightarrow \infty
    \end{equation*}
    Using a similar argument, for $\textrm{II}$ we can write,
    \begin{equation}\label{eqn_ii9}
        \textrm{II} \le \prob\left(
                        2n^{-1} \| \theta^{i,pa_i}_{pa_i} - \tilde{\theta}^{i,pa_i}_{pa_i} \|_{\infty}
                        | \mathbf{1}_{pa_i} \T{\mathcal{X}}_{pa_i} \mathcal{R}_j | \ge \lambda^0 / 3
                    \right).
    \end{equation}
    Since columns of $\mathcal{X}_{pa_i}$ correspond to samples from normal random variables with mean zero and are all independent of $\mathcal{R}_j$, it suffices to show that there exists $c_{(\textrm{II})} > 0$ such that for all $j \in \underline{i-1} \backslash pa_i$ and for all $k \in pa_i$,
    \begin{equation}\label{eqn_ii10}
        \prob\left(
            2n^{-1} \| \theta^{i,pa_i}_{pa_i} - \tilde{\theta}^{i,pa_i}_{pa_i} \|_{\infty} \card{(pa_i)} | \T{\mathcal{X}_k} \mathcal{R}_j | \ge \lambda^0 / 3 \right)
            = O\left\{ \exp{(-c_{(\textrm{II})} n^\zeta)} \right\} \thickspace \text{ as } n \rightarrow \infty
    \end{equation}
    By (\ref{eqn_ii8}) and (A-$0^\prime$), the random variable on the left hand side of (\ref{eqn_ii10}) is stochastically smaller than $2n^{-1}|\mathcal{X}_k \mathcal{R}_j|$. By independence of $X_k$ and $R_j$, $\E(X_k R_j) = 0$. Also using Gaussianity of both $X_k$ and $R_j$, there exists $g < \infty$ such that $\E\{\exp{(|X_k R_j|)}\} \le g$. Since $\lambda^0 = O\{{(\log p/n)}^{1/2}\}$, by Bernstein's inequality \citep{vandervaart1996wca}, $\prob(2n^{-1}|\mathcal{X}_k \mathcal{R}_j| > \lambda^0/3) \le \exp(-c_{(\textrm{II})} n^\zeta)$ for some $c_{(\textrm{II})} > 0$ and hence (\ref{eqn_ii10}) is satisfied.
    Finally, for $\textrm{III}$ we have
    \begin{equation}\label{eqn_ii11}
     \prob\left\{ |-2n^{-1} \T{( \mathcal{X}_{pa_i} \theta^{j,pa_i}_{pa_i} + \mathcal{R}_i)} \mathcal{R}_j | \ge \lambda^0 / 3 \right\}
     = \prob\left\{ |-2n^{-1} \T{ \mathcal{X}_{i} } \mathcal{R}_j | \ge \lambda^0 / 3 \right\}
    \end{equation}
    and using the Bernstein's inequality we conclude that there exists $c_{(\textrm{III})} > 0$ such that for all $j \in \underline{i-1} \backslash pa_i$ and for all $k \in pa_i$, $\textrm{III} = O\left\{ \exp{(-c_{(\textrm{III})} n^\zeta)} \right\} \thickspace \text{ as } n \rightarrow \infty$. The proof of (ii) is then complete by taking $c_{(ii)}$ to be the minimum of $c_{(\textrm{I})}, \ldots, c_{(\textrm{III})}$.
    %%%%%%%%%%
    %% (iii)
    %%%%%%%%%%

    To prove (iii), note that $\prob \left( pa_i \subseteq \hat{pa}_i \right) = 1 - \prob \left(\exists j \in pa_i : \hat{\theta}^i_j = 0 \right)$, and let $ \mathcal{E} = \left\{ \max_{k \in \underline{i-1} \backslash pa_i}{| G_j(\hat{\theta}^{i,pa_i}) |} < \lambda w_{ij} \right\}$.

    It follows from an argument similar to the proof of (ii) that conditional on $\mathcal{E}$, $\hat{\theta}^{i,pa_i}$ is an almost sure unique solution of the unrestricted adaptive lasso problem (\ref{eqnThetaHat}) with $\mathcal{I} = \underline{i-1}$. Therefore,
    \[
        \prob \left(\exists j \in pa_i : \hat{\theta}^i_j = 0 \right) \le \prob \left(\exists j \in pa_i : \hat{\theta}^i_j = 0 \right) + \prob \left(\mathcal{E}^c \right).
    \]
    From (i), there exists a $c_1 > 0$ such that $\prob \left(\exists j \in pa_i : \hat{\theta}^i_j = 0 \right) = O(\exp{(-c_1n^\zeta)})$ and it was shown in (ii) that $\prob \left(\mathcal{E}^c \right) = O\{\exp{(-c_2n^\zeta)}\}$ for some $c_2 > 0$. Thus (iii) follows from Bonferroni's inequality.

    Finally, it is easy to see that since $p = O(n^a)$, the claim in (iv) also follows from (ii) and (iii) and Bonferroni's inequality.
\end{proof}
%%%%%%%%%%%%%%%%%%%%%%%%%%%%%%%%%%%%%%%%%%
\begin{proof}[of Theorem \ref{Thm_tuning}]
    We first show that if $AN_i \cap AN_j = \emptyset$, then $i$ and $j$ are independent. Since $\Sigma = \Lambda \T{\Lambda}$ and $\Lambda$ is lower triangular,
    \begin{equation}\label{eqnThm3_1}
        \Sigma_{ij} = \sum_{k=1}^{\min{(i,j)}}{\Lambda_{ik}\Lambda_{jk}}
    \end{equation}
    We assume without loss of generality that $i < j$. The argument for the case $j > i$ is similar. Suppose for all $k = 1, \ldots, i$ that $\Lambda_{ik} = 0$ or $\Lambda_{jk} = 0$, then by (\ref{eqnThm3_1}) $i$ and $j$ are independent. However, by Lemma \ref{lmLambdaA}, $\Lambda_{jk}$ is the influence of $k$th node on $j$, and this is zero only if there is no path from $k$ to $j$. Clearly, if $i$ is an ancestor of $j$, we have $\Sigma_{ij} \ne 0$. On the other hand, if there is no node $k \in \underline{i-1}$ such that $k$ influences both $i$ and $j$, (i.e. $k$ is a common ancestor of $i$ and $j$) then for all $k = 1, \ldots, i$ we have ${\Lambda_{ik}\Lambda_{jk}} = 0$ and the claim follows.

    Using Bonferroni's inequality twice and Lemma \ref{lemma_AlassoKKT}, we get
    \begin{eqnarray*}
        \prob( \exists i \in V: \hat{AN}_i \nsubseteq AN_i )
        &\le& p \max_{i \in V}{ \prob\left( \exists j \in \underline{i-1} \backslash AN_i : j \in \hat{pa}_i \right) } \\
        &\le& p (i-1) \max_{i \in V, j \in \underline{i-1} \backslash AN_i}{ \prob\left( j \in \hat{pa}_i \right) } \\
        &\le& p (i-1) \max_{i \in V, j \in \underline{i-1} \backslash AN_i}{ \prob\left\{ | G_j(\hat{\theta}^{i,AN_i}) | \ge \lambda w_{ij} \right\}. }
    \end{eqnarray*}
%    \noindent Hence it suffices to show that $
%    (i-1)p \max_{i \in V, j \in \underline{i-1} \backslash AN_i}{ \prob\left( | G_j(\hat{\theta}^{i,AN_i}) | \ge \lambda w_{ij} \right) } \le \alpha$.
%    \begin{equation}\label{eqnThm3_2}
%        (i-1)p \max_{i \in V, j \in \underline{i-1} \backslash AN_i}{
%                            \prob\left( | G_j(\hat{\theta}^{i,AN_i}) | \ge \lambda w_{ij} \right) } \le \alpha
%    \end{equation}
    However, by definition $w_{ij} \ge 1$, and hence it suffices to show that,
    \begin{equation}\label{eqnThm3_2}
        (i-1)p \max_{i \in V, j \in \underline{i-1} \backslash AN_i}{
                            \prob\left\{ | G_j(\hat{\theta}^{i,AN_i}) | \ge \lambda \right\} } \le \alpha.
    \end{equation}
    Note that $G_j(\hat{\theta}^{i,AN_i}) = -2n^{-1} \T{\mathcal{X}}_j ( \mathcal{X}_i - \mathcal{X}\hat{\theta}^{i,AN_i} )$ and $X_j$ is independent of $X_k$ for all $k \in AN_i$. Therefore, conditional on $\mathcal{X}_{AN_i}$, $G_j(\hat{\theta}^{i,AN_i}) \sim (0,4R^2/n)$, where $R^2 = n^{-1}\| \mathcal{X}_i - \mathcal{X}\hat{\theta}^{i,AN_i} \|_2^2 \le n^{-1}\| \mathcal{X}_i \|_2^2 = 1$, by definition of $\hat{\theta}^{i,AN_i}$ and the fact that columns of the data matrix are scaled.

    It follows that for all $j \in \underline{i-1} \backslash AN_i$, $\prob\left\{ | G_j(\hat{\theta}^{i,AN_i}) | \ge \lambda \mid \mathcal{X}_{AN_i} \right\} \le 2 \{ 1-\Phi(n^{1/2} \lambda / 2) \} $, where $\Phi$ is the cumulative distribution function for standard normal random variable. Using the choice of $\lambda$ proposed in (\ref{eqnLambda}), we get $
        \prob\left\{ | G_j(\hat{\theta}^{i,AN_i}) | \ge \lambda \mid \mathcal{X}_{AN_i} \right\} \le \frac{\alpha}{(i-1)p}
    $, and the result follows.
\end{proof}
%%%%%%%%%%%%%%%%%%%%%%%%%%

%%%%%%%%%%%%%%%%%%%%%%%%%%%
\bibliography{ShojaieBib}
%%%%%%%%%%%%%%%%%%%%%%%%%%%

%%%%%%%%%%%%%%%%%%%%%%%%%%%%%%%%%%%%%%%%%%%%%%%%%%%%%%%%%%%%%%%%%%%%%%%%%%%%%
\end{document}